\documentclass[letterpaper]{article}
\usepackage{aaai25}
\usepackage{hyperref}

\usepackage{xcolor}
\definecolor{navy}{rgb}{0.0, 0.0, 0.5}

\hypersetup{
  citecolor  = black,
  urlcolor   = navy,
  colorlinks = true,
}

\usepackage{amssymb,amsthm,amsmath}
\newtheorem{theorem}{Theorem}
\newtheorem{example}{Example}

\newtheorem{lemma}{Lemma}
\newtheorem{claim}{Claim}
\newtheorem{remark}{Remark}
\newtheorem{definition}{Definition}
\newtheorem{corollary}{Corollary}

\usepackage{times}
\usepackage{helvet}
\usepackage{courier}
\usepackage{graphicx}
\usepackage{natbib}
\usepackage{caption}
\nocopyright

\usepackage[frozencache]{minted}
\usepackage{thm-restate}

\title{Revelations: A Decidable Class of POMDPs with Omega-Regular Objectives}
\author {
    Marius Belly\textsuperscript{\rm 1},
    Nathanaël Fijalkow\textsuperscript{\rm 1},
    Hugo Gimbert\textsuperscript{\rm 1},\\
    Florian Horn\textsuperscript{\rm 2},
    Guillermo A. Pérez\textsuperscript{\rm 3},
    Pierre Vandenhove\textsuperscript{\rm 1}\thanks{Authors are listed in alphabetical order.}
}
\affiliations {
    \textsuperscript{\rm 1}CNRS, LaBRI, Universit\'e de Bordeaux, France\\
    \textsuperscript{\rm 2}CNRS, IRIF, Universit\'e de Paris, France\\
    \textsuperscript{\rm 3}University of Antwerp -- Flanders Make, Antwerp, Belgium
}

\usepackage{tikz}
\usepackage{cleveref}
\usepackage{booktabs}
\usepackage{listings}
\input{commands.tex}

\begin{document}

\maketitle

\begin{abstract}
Partially observable Markov decision processes (POMDPs) form a prominent model for uncertainty in sequential decision making. We are interested in constructing algorithms with theoretical guarantees to determine whether the agent has a strategy ensuring a given specification with probability~$1$. This well-studied problem is known to be undecidable already for very simple omega-regular objectives, because of the difficulty of reasoning on uncertain events.
We introduce a revelation mechanism which restricts information loss by requiring that almost surely the agent has eventually full information of the current state. Our main technical results are to construct exact algorithms for two classes of POMDPs called \emph{weakly} and \emph{strongly revealing}. Importantly, the decidable cases reduce to the analysis of a finite belief-support Markov decision process. This yields a conceptually simple and exact algorithm for a large class of POMDPs.
\end{abstract}

\begin{links}
    \link{Code}{https://github.com/gaperez64/pomdps-reveal}
\end{links}

\section{Introduction}
Partially observable Markov decision processes (POMDPs) form a prominent model for uncertainty in sequential decision making.
They were defined in the 1960s~\cite{Astrom:1965} for operations research and introduced in artificial intelligence by the seminal paper of~\citet{Kaelbling.Littman.ea:1998}.
We consider POMDPs from a model-based point of view common in planning and in formal methods. Our goal is to construct exact (as opposed to approximate) algorithms that take as an input a complete description of the POMDP and construct a strategy ensuring a given specification.
A long line of work has established that most formulations of this problem are
undecidable.
For instance, even in the extreme case where the agent has no information and
the goal is to reach a target state with arbitrarily high probability, complex
convergence phenomena occur, implying strong undecidability
results~\cite{Madani.Hanks.ea:2003,Gimbert.Oualhadj:2010,Fijalkow:2017}.

In this work, we are interested in constructing \emph{almost-sure
strategies}, meaning strategies ensuring their specifications with probability~$1$.
We consider the class of omega-regular objectives (all expressible as \emph{parity objectives}), which is a robust class including properties expressible in Linear Temporal Logic~\cite{Pnueli:1977,Giacomo.Vardi:2013}.
Determining whether there exists an almost-sure strategy against the subclass of CoB{\"u}chi objectives (requiring to avoid a target from some point onwards) is undecidable~\cite{Chatterjee.Chmelik.ea:2016,Bertrand.Genest.ea:2017}.
There is a vast body of work towards approximate and practical
solutions: for instance, using interpolation in the belief space~\cite{Lovejoy:1991}, approximation of the value function~\cite{Hauskrecht:2000}, or Monte Carlo tree search approaches~\cite{Silver.Veness:2010}.
This is orthogonal to the current paper since we focus on exact algorithms.

\textbf{Our starting point} is a simple approach to construct almost-sure strategies:
from the POMDP, we build a Markov decision process (MDP) whose states are \emph{supports of the beliefs} of the POMDP.
In other words, we store information about which states we can be in, but abstract away the probabilities.
The \emph{belief-support MDP} serves as a finite abstraction of the POMDP; one could expect that there exists an almost-sure strategy in the POMDP if and only if there exists one in the corresponding belief-support MDP.
Unfortunately, this abstraction is neither sound nor complete; we present a simple counterexample in Figure~\ref{fig:simpleUnsoundIncomplete}.

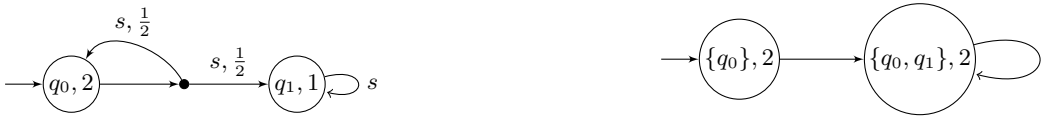
\begin{figure*}[t]
	\centering
	\begin{minipage}{.5\textwidth}
		\centering
		\begin{tikzpicture}[every node/.style={font=\small,inner sep=1pt}]
			\draw ($(0,0)$) node[rond] (q1) {$\state_0, 2$};
			\draw ($(q1.west)-(0.5,0)$) edge[-latex'] (q1);
			\draw ($(q1)+(3,0)$) node[rond] (q2) {$\state_1, 1$};
			\draw ($(q1)!.5!(q2)$) node[dot] (qmid) {};
			\draw (q2) edge[-latex',out=-30,in=30,loop right] node[right=2pt] {$\sig$} (q2);
			\draw (q1) edge[-latex'] (qmid);
			\draw (qmid) edge[-latex',out=120,in=60] node[above=2pt] {$\sig, \frac{1}{2}$} (q1);
			\draw (qmid) edge[-latex'] node[above=2pt] {$\sig, \frac{1}{2}$} (q2);
		\end{tikzpicture}
	\end{minipage}%
	\begin{minipage}{.5\textwidth}
		\centering
		\begin{tikzpicture}[every node/.style={font=\small,inner sep=1pt}]
			\draw ($(0,0)$) node[rond] (q1) {$\set{\state_0}, 2$};
			\draw ($(q1.west)-(0.5,0)$) edge[-latex'] (q1);
			\draw ($(q1)+(2.4,0)$) node[rond] (q2) {$\set{\state_0, \state_1}, 2$};
			\draw (q2) edge[-latex',out=-30,in=30,loop right] (q2);
			\draw (q1) edge[-latex'] (q2);
		\end{tikzpicture}
	\end{minipage}
	\caption{We consider the POMDP on the LHS: there is a single signal $\sig$, so no information is ever given about the exact state we are in (a behavior the revelation mechanisms forbid!).
	Yet, almost surely, we reach $\state_1$.
	The \emph{priorities} indicated on states constitute a parity condition inducing the objective ``eventually never visiting $\state_1$'', which clearly cannot be ensured almost surely.
	We represent the \emph{belief-support MDP} on the RHS: the two states are $\{\state_0\}$ and $\set{\state_0, \state_1}$, and only the state $\set{\state_0, \state_1}$ is visited infinitely often.
	To assign priorities to the states of this MDP, there are two natural candidates: ``maximal priority semantics'' and ``minimal priority semantics'', meaning that we assign either the maximal or minimal priority from the states in the belief support.
	In this figure, we use the maximal priority semantics: the priority of $\set{\state_0, \state_1}$ is thus $2$, so the belief-support MDP is winning. This means that the analysis of the belief-support MDP is not sound in general.
	By tweaking the priorities in this example, one can show that both priority semantics are neither sound nor complete.}
	\label{fig:simpleUnsoundIncomplete}
\end{figure*}

The fundamental question we ask in this paper is whether \textbf{there are natural sufficient conditions which make the belief-support abstraction correct}.
Conceptually, the failure of this abstraction is due to information loss and its accumulation over time.

We introduce a \textbf{revelation mechanism} which restricts information loss by requiring that, almost surely, the agent has eventually full information of the current state.
Intuitively, by forbidding information loss from accumulating for an unbounded amount
of time, the revelation mechanism removes the convergence issues leading to undecidability.
Practically, we conjecture that revelation is a commonly occurring phenomenon
in partial observability; a canonical example is systems with a
small probability of resetting infinitely often, and where this reset is observable.
We leave to future work to investigate this question further.
Other approaches to restrict information loss have been proposed; we refer to the related works (Section~\ref{sec:relatedWorks}) for an additional discussion.

\begin{figure}[t]
\centering
\includegraphics[width=0.9\columnwidth]{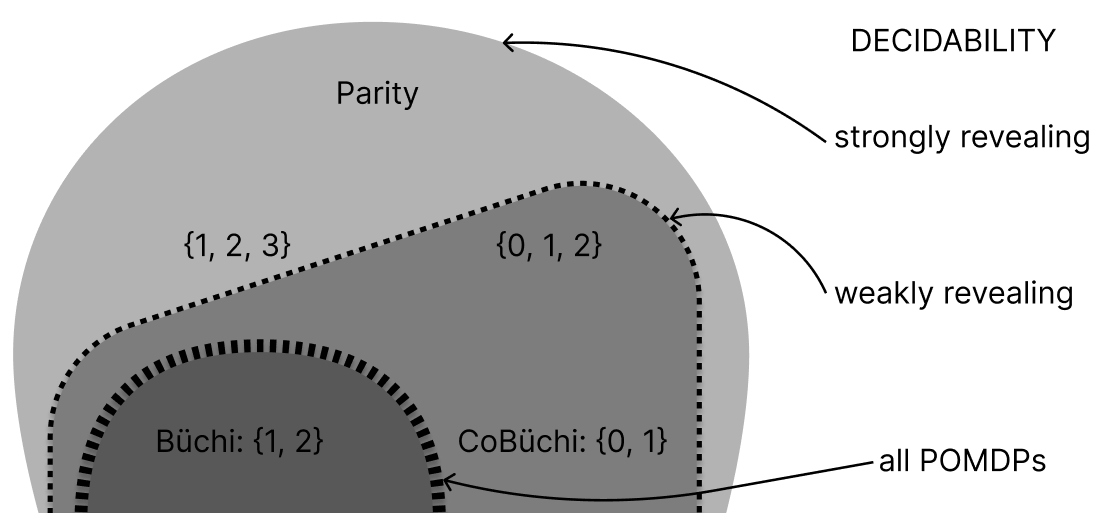}
\caption{Summary of our results: decidable subclasses of the \emph{parity} objective depending on the revelation mechanism.}
\label{fig:summary}
\end{figure}

\begin{figure}[t]
  \centering
  \includegraphics[width=0.81\columnwidth]{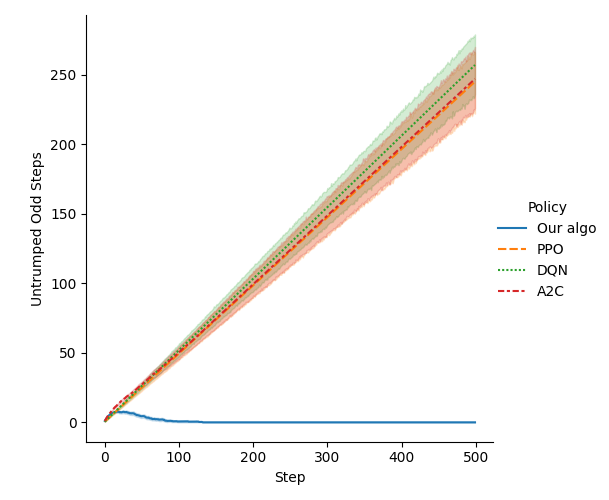}
  \caption{Omega-regular specifications have
  a natural interpretation in terms of \emph{bad} events that must all be
\emph{trumped} by future \emph{good} events. Along a simulation of the POMDP,
one can keep track of the number of steps from the last bad
event that has not yet been trumped (i.e., lower is better). Here, we depict
this value, per step (from $1$ to $500$) over $500$ simulations of a revealing
version of the classical tiger POMDP~\cite{tiger94}. A2C, DQN, and PPO
	are (\texttt{MlpPolicy}) strategies obtained from the \emph{stable-baselines} library~\cite{Raffin.Hill.ea:2021}, trained
(for a total of $10$k time steps) with default parameter values using a simple
reward scheme: a good event yields a reward of~$100$; a bad one,~$-1$. In the
simulations, the trained models are queried for deterministic action
predictions.
The example used will be discussed in Section~\ref{sec:stronglyRevealing}, Example~\ref{ex:revealingTiger}.}
\label{fig:drl-vs-ours}
\end{figure}

\paragraph*{Our contributions.}
\begin{itemize}
	\item We study two properties of POMDPs based on the revelation mechanism, called \emph{weak} and \emph{strong revelations}.
	\item We obtain decidability (and undecidability) results for both classes. Importantly, the decidable cases reduce to the analysis of the finite belief-support MDP.
	A summary of our contributions for POMDPs is provided in Figure~\ref{fig:summary}.
	We also briefly consider the class of \emph{two-player games of partial information}, to show that our revealing mechanisms do not suffice for decidability on this larger class.
	\item We provide a simple implementation of the algorithm as a proof of
    concept. We provide a comparison
    between our algorithm and off-the-shelf deep reinforcement learning (DRL)
    trained via an observation wrapper.
	As we will show in the paper, the MDP induced by the belief supports carries sufficient information to play in revealing POMDPs; hence, we used a wrapper implementing a subset construction on the fly to generate the current belief support, and focused on algorithms intended for MDPs.
    Spending moderate effort on reward engineering and hyperparameter tuning, we have been unable to match the performance of our algorithm using such DRL approaches (see Figure~\ref{fig:drl-vs-ours}).
\end{itemize}
This yields a conceptually simple and exact algorithm for a large class of POMDPs.
The importance of our results can be appreciated by the following remark:
instead of a subclass of POMDPs, the revelation mechanism can be seen as new
semantics for \emph{all} POMDPs. In that sense, we obtain decidability results
for an \emph{optimistic} semantics of POMDPs which, to the best of our
knowledge, has not been done before. We refer to Section~\ref{sec:optimistic_semantics} for more details on this point of view.

\paragraph*{Outline.}
We define our stochastic models in Section~\ref{sec:preliminaries}.
We then introduce the belief-support MDP in Section~\ref{sec:beliefSupportMDP}.
Sections~\ref{sec:weaklyRevealing} and~\ref{sec:stronglyRevealing} are devoted to our revealing mechanisms and to our (un)decidability results.
We conclude with additional related works and perspectives in Sections~\ref{sec:relatedWorks} and~\ref{sec:perspectives}.
To highlight our results, we defer most proofs to the appendix.

This article constitutes the extended version (including all proofs, additional examples and explanations) of an article with the same name accepted in the proceedings of the AAAI'25 conference~\cite{aaaiPaper}.

\section{Preliminaries}
\label{sec:preliminaries}
A \emph{(discrete) probability distribution} on a finite set $X$ is a function $d\colon X \to [0, 1]$ such that $\sum_{x\in X} d(x) = 1$.
The set of all probability distributions on $X$ is denoted $\dist{X}$.
The \emph{support} $\supp{d}$ of a probability distribution $d$ is the set $\set{x\in X \mid d(x) > 0}$.
We let $\card{X}$ denote the number of elements in a set $X$.

\subsection{POMDPs}

A \emph{partially observable Markov decision process} (POMDP) is a tuple $\pomdp = \pomdpFull$ such that $\states$ is a finite set of \emph{states}, $\actions$ is a finite set of \emph{actions}, $\signals$ is a finite state of \emph{signals}, $\transitions\colon \states\times\actions \to \dist{\signals\times\states}$ is the \emph{transition function}, and $\initState\in\states$ is an \emph{initial state}.

A \emph{play} of a POMDP $\pomdp = \pomdpFull$ is an infinite sequence $\play = \state_0\act_1\sig_1\state_1\act_2\sig_2\ldots \in (\states\cdot\actions\cdot\signals)^\omega$ such that, for all $i\ge 0$, $\transitions(\state_i, \act_{i+1})(\sig_{i+1}, \state_{i+1}) > 0$.
A \emph{history}~$\hist$ of a POMDP is a finite prefix of a play ending in a state (it is an element of $(\states\cdot\actions\cdot\signals)^*\cdot\states$).
If $\hist = \state_0\act_1\sig_1\state_1\ldots\act_n\sig_n\state_n$, we write $\last(\hist)$ for $\state_n$.
In practice, states are not fully observable; we define an \emph{observable history} as the projection of a history to the subsequence in $(\actions\cdot\signals)^*$.
We write $\obs(\hist)$ for the observable history derived from a history $\hist$, i.e., the same sequence with the states removed.

For $\state$ a state of a POMDP $\pomdp = \pomdpFull$, we define $\pomdp^\state$ to be the POMDP $(\states, \actions, \signals, \transitions, \state)$ with only a change of initial state.
We let $\leastProb_\pomdp = \min \set{\transitions(\state,\act)(\sig, \state') \mid \state,\state'\in\states,\, \act\in\actions,\, \sig\in\signals,\, \text{and}\ \transitions(\state,\act)(\sig, \state') > 0}$ denote the least non-zero probability occurring in $\pomdp$.

A \emph{Markov decision process} (MDP) is a tuple $\mdp = \mdpFull$ where $\transitions\colon \states\times\actions \to \dist{\states}$.
Formally, an MDP $\mdp = \mdpFull$ can be seen as a POMDP $\pomdp = \pomdpFull$ such that $\signals = \set{\sig_{\state} \mid \state\in\states}$ and for all $\state, \state', \state''\in\states$ and $\act\in\actions$, $\transitions(\state, \act)(\sig_{\state''}, \state') > 0$ if and only if $\state' = \state''$.
In practice, it means that the last signal always uniquely determines the current state.
MDPs have ``complete observation'', whereas POMDPs have ``partial observation''.
For a POMDP $\pomdp = \pomdpFull$, we define the \emph{underlying MDP of $\pomdp$} to be the MDP~$(\states, \actions, \transitions', \initState)$ with $\transitions'(\state, \act)(\state') = \sum_{\sig\in\signals} \transitions(\state, \act)(\sig, \state')$.

\begin{remark}
	The observable information in POMDPs is here provided through \emph{signals} that appear along transitions.
	This contrasts with state-based \emph{observations} that partition the state space, which are also frequently used to model POMDPs.
	Both models are polynomially equivalent: a POMDP with observations can be transformed into an equivalent POMDP with signals on the same state space, while the converse requires an increase of the state space linear in $\card{\signals}$.
	Both choices are convenient, but using signals make the definition of \emph{strongly revealing} (Definition~\ref{def:stronglyRevealing}) more natural, which is why we opted for this convention.
\end{remark}

\paragraph*{Strategies.}
Let $\pomdp = \pomdpFull$ be a POMDP.
An \emph{(observation-based) strategy in $\pomdp$} is a function that makes decisions based on the current observable history, i.e., it is a function $\strat\colon (\actions\cdot\signals)^* \to \dist{\actions}$.
We can define strategies in MDPs similarly (i.e., assuming that $\signals$ gives the information of the current state), but we assume for convenience that a strategy is a function $\strat\colon (\actions\cdot\states)^* \to \dist{\actions}$ in this case.
An observable history $\act_1\sig_1\ldots\act_n\sig_n$ is \emph{consistent with a strategy $\strat$} if for all $1\le i < n$, $\strat(\act_1\sig_1\ldots\act_i\sig_i)(\act_{i+1}) > 0$.

A strategy $\strat$ is \emph{pure} if for all observable histories $\sigHist\in(\actions\cdot\signals)^*$, $\strat(\sigHist)$ is a Dirac distribution;
in other words, if $\strat$ is a function $(\actions\cdot\signals)^* \to \actions$.
We let $\strats{\pomdp}$ denote the set of strategies in POMDP $\pomdp$ and $\stratsPure{\pomdp}$ denote the set of pure strategies in $\pomdp$.

For an MDP $\mdp$, a strategy $\strat$ in $\mdp$ is \emph{memoryless} if its decisions are only based on the current state: i.e., if for all histories $\hist_1, \hist_2$, $\last(\hist_1) = \last(\hist_2)$ implies $\strat(\hist_1) = \strat(\hist_2)$.
We only define the memoryless notion for MDPs.

\paragraph*{Probability measure induced by a strategy.}
Let $\pomdp = \pomdpFull$ be a POMDP.
For a history $\hist$ of $\pomdp$, we define $\Cyl(\hist)$ (the \emph{cylinder of $\hist$}) to be the set of all plays starting with $\hist$, i.e., $\hist(\actions\cdot\signals\cdot\states)^\omega$.
Given a strategy $\strat$, we can define a probability measure~$\prob{\strat}{\pomdp}{\cdot}$ on infinite plays.
This function is naturally defined over cylinders by induction.
We define $\prob{\strat}{\pomdp}{\Cyl(\initState)} = 1$, and $\prob{\strat}{\pomdp}{\Cyl(\state)} = 0$ for $\state\in\states$, $\state\neq\initState$.
For a history $\hist = \hist'\act\sig\state$, we define $\prob{\strat}{\pomdp}{\Cyl(\hist)} = \prob{\strat}{\pomdp}{\Cyl(\hist')}\cdot\strat(\obs(\hist'))(\act)\cdot\transitions(\last(\hist'), \act)(\sig, \state)$.
By Ionescu-Tulcea extension theorem~\cite{Klenke:2007}, this function can be uniquely extended to a probability distribution $\prob{\strat}{\pomdp}{\cdot}$ over the Borel sets of infinite plays induced by all cylinders.

We also use this probability distribution to measure sets of infinite sequences in~$\states^\omega$, by associating a set $\objective \subseteq \states^\omega$ with the set $\bigcup_{\state_0\state_1\ldots\in\objective} \state_0\actions\signals\state_1\actions\signals\state_2\ldots \subseteq (\states\times\actions\times\signals)^\omega$.
Similarly, we use this probability distribution to measure events based only on signals, by associating a set $\sigEvents \subseteq \signals^\omega$ with the set $\bigcup_{\sig_1\sig_2\ldots\in\sigEvents} \states\actions\sig_1\states\actions\sig_2\states\ldots \subseteq (\states\times\actions\times\signals)^\omega$.

\paragraph*{Objectives.}
Let $\pomdp = \pomdpFull$ be a POMDP.
An \emph{objective} $\objective\subseteq \states^\omega$ is a measurable set of infinite sequences of states.
Note that observing an infinite sequence of signals (but not the states) may not always be sufficient to determine whether a play satisfies an objective.

Given a set $\reach\subseteq \states$, the \emph{reachability objective} $\Reach(\reach) = \set{\state_0\state_1\ldots \in \states^\omega\mid \exists i\ge 0, \state_i\in\reach}$ is the set of plays that visit a state in $\reach$ at least once.
For $k\in\IN$, we write $\Reach^{\le k}(\reach) = \set{\state_0\state_1\ldots \in \states^\omega\mid \exists i, 0\le i\le k, \state_i\in\reach}$ for the set of plays that reach $\reach$ in at most $k$ steps.
Given a set $\safe\subseteq \states$, the \emph{safety objective} $\Safety(\safe)$ is the set of plays that never visit any state in $\safe$.

Given a \emph{priority function} $\pri\colon \states \to \set{0, \ldots, d}$ (where $d\in\IN$), the \emph{parity objective} $\Parity(\pri) = \set{\state_0\state_1\ldots \in \states^\omega\mid \text{$\limsup_{i\ge 0} \pri(\q_i)$ is even}}$ is the set of infinite plays whose highest priority seen infinitely often is even.
A \emph{B\"uchi objective} is a parity objective $\Parity(\pri)$ such that $\pri\colon\states\to\set{1, 2}$, and a \emph{CoB\"uchi objective} is a parity objective $\Parity(\pri)$ such that $\pri\colon\states\to\set{0, 1}$.
For $\states'\subseteq\states$, we write $\Buchi(\states')$ for the set of infinite plays that visit $\states'$ infinitely often~;
it is equal to $\Parity(\pri)$ for the priority function $\pri$ such that $\pri(\state) = 2$ if $\state\in\states'$, and $\pri(\state) = 1$ otherwise.
We also write $\coBuchi(\states')$ for the set of infinite plays that visit $\states'$ only finitely often~; it is equal to $\Parity(\pri)$ for the priority function $\pri$ such that $\pri(\state) = 1$ if $\state\in\states'$, and $\pri(\state) = 0$ otherwise.

For an objective $\objective$, a strategy $\strat$ is \emph{almost sure} if $\prob{\strat}{\pomdp}{\objective} = 1$, and is \emph{positively winning} if $\prob{\strat}{\pomdp}{\objective} > 0$.
We say that an objective $\objective$ has \emph{value~$1$} in a POMDP $\pomdp$ if $\sup_{\strat\in\strats{\pomdp}} \prob{\strat}{\pomdp}{\objective} = 1$.

\subsection{Beliefs and belief supports}

Let $\pomdp = \pomdpFull$ be a POMDP.
A \emph{belief} $\belief\in\dist{\states}$ is a probability distribution on $\states$.
A \emph{belief support} $\beliefSupp \in \powerSet{\states} \setminus \set{\emptyset}$ is the support of a belief.
For brevity, we write $\powerSetNonEmpty{\states}$ for $\powerSet{\states} \setminus \set{\emptyset}$.
At every step, beliefs and belief supports can be updated when playing an action and observing a signal.
We show how to do so for belief supports: we define a function $\beliefUpd\colon \powerSetNonEmpty{\states} \times \actions\times\signals \to \powerSetNonEmpty{\states}$ that updates the belief support.
For $\beliefSupp\in\powerSetNonEmpty{\states}$, $\act\in\actions$, $\sig\in\signals$, we define
$
    \beliefUpd(\beliefSupp, \act, \sig) = \{\state'\in\states \mid \exists \state\in\beliefSupp, \transitions(\state, \act)(\sig, \state') > 0\}.
$
We extend this function in a natural way to a function $\beliefUpdStar\colon \powerSetNonEmpty{\states}\times (\actions\cdot\signals)^* \to \powerSetNonEmpty{\states}$.
Objectives $\Reach(B)$ and $\Buchi(B)$ can be naturally extended to sets of belief supports $B\subseteq \powerSetNonEmpty{\states}$ (see Appendix~\ref{app:beliefSupportMDP}).

Beliefs carry more information than belief supports, as they contain the exact probability of being in a particular state, while belief supports only contain the qualitative information of the possible current states.
Observe that when the belief support is a singleton (i.e., $\beliefSupp = \set{\state}$ for some $\state\in\states$), knowing the precise belief does not yield more information than knowing the belief support, as all the probability mass is in one of the states.
Our ``revealing'' restrictions on POMDPs defined later will exploit this fact.

\section{The belief-support MDP}
\label{sec:beliefSupportMDP}
For a POMDP $\pomdp = \pomdpFull$, the \emph{belief-support MDP of $\pomdp$} is the MDP $\beliefMDP = (\powerSetNonEmpty{\states}, \actions, \transitionsMDP, \{\initState\})$ where for $\beliefSupp, \beliefSupp'\in\powerSetNonEmpty{\states}$ and $\act\in\actions$,
$
	\transitionsMDP(\beliefSupp, \act)(\beliefSupp') > 0
$
if and only if there is $\sig\in\signals$ such that $\beliefUpd(\beliefSupp, \act, \sig) = \beliefSupp'$.
We assume the distribution to be uniform over successors with positive probability.

We can show that for multiple simple objectives, the POMDP and its belief-support MDP behave in a similar way.
For example, sets of belief supports that can be reached with a positive probability are the same in the POMDP and its belief-support MDP (Appendix~\ref{app:beliefSupportMDP}, Lemma~\ref{lem:bijectionStrategies}); if a set of belief supports is reachable almost surely in the POMDP, it is also the case in the belief-support MDP (Appendix~\ref{app:beliefSupportMDP}, Lemma~\ref{lem:bijectionStrategiesAS}).

There is a natural way to lift a strategy in the belief-support MDP to a strategy in the POMDP.
We define a notation to go from a sequence of signals to the induced sequence of belief supports.
Let $\hist = \act_1\sig_1\ldots\act_n\sig_n\in(\actions\cdot\signals)^*$ be a possible observable history in $\pomdp$.
For $1\le i\le n$, let $\beliefSupp_i = \beliefUpdStar(\{\initState\}, \act_1\sig_1\ldots\act_i\sig_i)$ be the belief support after $i$ steps.
We define $B_\hist$ to be the history $\act_1\beliefSupp_1\ldots\act_n\beliefSupp_n$ of $\beliefMDP$.
Let $\strat_\beliefUpd \in \strats{\beliefMDP}$ be a strategy in the belief-support MDP of a POMDP $\pomdp$.
We define a strategy $\liftStrat$ in $\pomdp$ derived from the strategy $\strat_\beliefUpd$: for $\hist\in(\actions\cdot\signals)^*$, we fix $\liftStrat(\hist) = \strat_\beliefUpd(B_\hist)$.

\section{Weakly revealing POMDPs}
\label{sec:weaklyRevealing}
We define here our first \emph{revealing} property for POMDPs, which requires that, infinitely often and almost surely, the current state can be deduced by looking at the previous sequence of signals.
Formally, we write $\singletons^\pomdp = \set{\set{\state} \mid \state\in\states}$ for the set of singleton belief supports of a POMDP $\pomdp = \pomdpFull$.
An observable history $\hist\in(\actions\cdot\signals)^*$ such that $\beliefUpdStar(\set{\state_0}, \hist)\in\singletons^\pomdp$ is called a \emph{revelation}.

\begin{definition}[Weakly revealing] \label{def:weaklyRevealing}
	A POMDP $\pomdp$ is \emph{weakly revealing} if, for all strategies $\strat\in\strats{\pomdp}$, we have $\prob{\strat}{\pomdp}{\Buchi(\singletons^\pomdp)} = 1$; i.e., for all strategies, infinitely many revelations occur almost surely.
\end{definition}

In particular, POMDPs that ``reset'' infinitely often, and whose reset can be observed with a dedicated signal, are weakly revealing.
We will use one such example in Figure~\ref{fig:incompleteCE}.

One can give probabilistic bounds on the occurrence of a revelation for a weakly revealing POMDP (see Lemma~\ref{lem:pomdpReach} in Appendix~\ref{app:boundsReachability} with $\reach = \singletons^\pomdp$): starting from any reachable belief, a revelation occurs within $2^{\card{\states}}-1$ steps with probability at least $\leastProb_\pomdp^{2^{\card{\states}}-1}$.

The bound is asymptotically tight: there is a weakly revealing POMDP with $n + 2$ states, $1$ action, and $n$ signals where we need at least $2^n - 1$ steps before observing a revelation with positive probability.
Details are provided in Example~\ref{ex:expLowerBound}, Appendix~\ref{app:weaklyRevealing}.

\subsection{Soundness of the belief-support MDP}

In this section, we show that, for \emph{weakly revealing} POMDPs, the existence of an almost-sure strategy in the belief-support MDP (with an adequate priority function) implies the existence of an almost-sure strategy in the POMDP.

For the priority function of the belief-support MDP, we consider the ``maximal priority'' semantics.
Formally, let $\pomdp = \pomdpFull$ be a POMDP, and $\beliefMDP$ be its belief-support MDP.
Let~$\pri\colon \states \to \set{0,\ldots,n}$ be a priority function on $\pomdp$, inducing the objective $\Parity(\pri)$.
We extend this function to the belief-support MDP: for $\beliefSupp \in \powerSetNonEmpty{\states}$, we define
\[
    \priMDP(\beliefSupp) = \max \set{\pri(\state) \mid \state\in\beliefSupp}.
\]

Without any assumption, the belief-support MDP may be unsound, already for B\"uchi objectives; there may be an almost-sure strategy in the belief-support MDP, but not in the POMDP.
An example illustrating this was given in Figure~\ref{fig:simpleUnsoundIncomplete}.
Surprisingly, it is sound for CoB\"uchi objectives without any assumption (see Lemma~\ref{lem:soundForCoBuchi} in Appendix~\ref{app:weaklyRevealing}).
Using ``$\max$'' (and not ``$\min$'') turns out to be the right choice in our setting.
Intuitively, under the right revealing assumptions and the right strategies, if a belief support is visited infinitely often, then all its states will be visited infinitely often, so the maximal priority of the belief support is the one that matters given the parity objective.
Without any assumption, both $\max$ and $\min$ are unsound and incomplete in general.

Under the weakly revealing semantics, almost-sure strategies of the belief-support MDP carry over to the POMDP for all parity objectives.
In other words, the analysis of the belief-support MDP is sound.
We recall that pure memoryless strategies suffice to reach the optimal value for parity objectives in MDPs~\cite{Chatterjee.Henzinger:2012}.

\begin{restatable}{proposition}{soundness} \label{prop:soundness}
    Let $\pomdp = \pomdpFull$ be a weakly revealing POMDP with priority function~$\pri$, and let $\beliefMDP$ be its belief-support MDP with priority function $\priMDP$.
    Assume there is an almost-sure strategy $\strat_\beliefUpd$ for $\Parity(\priMDP)$ in $\beliefMDP$; by~\cite{Chatterjee.Henzinger:2012}, we may assume $\strat_\beliefUpd$ to be pure and memoryless.
    Then, $\liftStrat$ is an almost-sure strategy for $\Parity(\pri)$ in $\pomdp$.
\end{restatable}

The complete proof is in Appendix~\ref{app:weaklyRevealing}.

\subsection{Decidability of parity for priorities \texorpdfstring{$0$}{0}, \texorpdfstring{$1$}{1}, and \texorpdfstring{$2$}{2}}

We show that the existence of an almost-sure strategy in a weakly revealing POMDP implies the existence of an almost-sure strategy in its belief-support MDP when priorities are in $\{0, 1, 2\}$.
This provides a converse to \Cref{prop:soundness} when priorities are restricted to $\{0, 1, 2\}$.
We will see that this is not the case for priorities in~$\set{1, 2, 3}$ in the next section; this result is therefore optimal w.r.t.\ the priority used.
We emphasize that parity objectives with priorities $\{0, 1, 2\}$ encompass both B\"uchi and CoB\"uchi objectives.
This result is false without the weakly revealing assumption; see the simple POMDP in Figure~\ref{fig:simpleUnsoundIncomplete}.
The proof is in Appendix~\ref{app:weaklyRevealing}.

\begin{restatable}{proposition}{completenessZeroOneTwo} \label{prop:completeness012}
    Let $\pomdp = \pomdpFull$ be a weakly revealing POMDP with priority function~$\pri$ with values in $\{0, 1, 2\}$.
    Let $\beliefMDP$ be its belief-support MDP with priority function $\priMDP$.
    If there is an almost-sure strategy for $\Parity(\pri)$ in~$\pomdp$, then there is an almost-sure strategy for $\Parity(\priMDP)$ in~$\beliefMDP$.
\end{restatable}

From the above, we deduce a complexity upper bound; a matching lower bound is proved in the appendix.
\begin{theorem} \label{thm:decidableWeak}
    The existence of an almost-sure strategy for parity objectives with priorities in $\{0, 1, 2\}$ in weakly revealing POMDPs is \EXPTIME-complete.
\end{theorem}
\begin{proof}
    The \EXPTIME algorithm is a consequence of the results from this section: by \Cref{prop:soundness} (soundness of the belief-support MDP) and \Cref{prop:completeness012} (completeness), we reduce the problem to the existence of an almost-sure strategy for a parity objective with priorities in $\set{0, 1, 2}$ in an MDP of size exponential in $\card{\states}$.
    The existence of an almost-sure strategy for parity objectives is decidable in polynomial time in MDPs~\cite[Theorem~10.127]{Baier.Katoen:2008}.
    \Cref{prop:soundness} also constructs an almost-sure strategy in~$\pomdp$.

    The \EXPTIME-hardness is proved in \Cref{prop:exptimeHardnessRevealing} (Appendix~\ref{app:stronglyRevealing}), already for CoB\"uchi objectives (i.e., with priorities in $\{0, 1\}$) and for the more restricted class of strongly revealing POMDPs.
\end{proof}

\begin{remark} \label{rem:memoryLowerBound}
    The algorithm also gives an upper bound on the size of the strategies for parity objectives with priorities in $\{0, 1, 2\}$ in weakly revealing POMDPs.
    As we reduce to the analysis of an exponential-size MDP and that memoryless strategies suffice for parity objectives in MDPs, given \Cref{prop:soundness}, it means that a strategy of exponential size suffices in the POMDP.
    We can also prove an exponential lower bound: see \Cref{ex:expLowerBound} (Appendix~\ref{app:weaklyRevealing}).
\end{remark}

\subsection{Undecidability of parity for priorities \texorpdfstring{$1$}{1}, \texorpdfstring{$2$}{2}, and \texorpdfstring{$3$}{3}} \label{subsec:weakUndecidable}
The previous section suggests that analyzing the belief-support MDP is a sound and complete approach for weakly revealing POMDPs with parity objectives with priorities in $\set{0, 1, 2}$.
One may wonder whether it is complete for any priority function, i.e., if the existence of an almost-sure strategy in a weakly revealing POMDP implies the existence of an almost-sure strategy in its belief-support MDP.
Unfortunately, this fails to hold in general, already for priority functions taking values in $\set{1, 2, 3}$.
We discuss one such example below.

\begin{figure*}[t]
	\centering
	\begin{minipage}{.5\textwidth}
		\centering
		\begin{tikzpicture}[every node/.style={font=\small,inner sep=1pt}]
			\draw ($(0,0)$) node[rond] (qinit) {$\state_0, 1$};
			\draw ($(qinit.west)-(0.5,0)$) edge[-latex'] (qinit);
			\draw ($(qinit)+(3.5,1.)$) node[rond] (q1) {$\state_1, 1$};
			\draw ($(qinit)+(3.5,-1.)$) node[rond] (q2) {$\state_1', 1$};
			\draw ($(q1)+(2,0)$) node[rond] (q3) {$\state_3, 3$};
			\draw ($(q2)+(2,0)$) node[rond] (q4) {$\state_2, 2$};
			\draw ($(q1)+(-1.5,-.5)$) node[dot] (qmid) {};
			\draw ($(q2)+(-1.5,0.5)$) node[dot] (qmid2) {};
			\draw (qinit) edge[-latex',out=45,in=160] (qmid);
			\draw (q1) edge[-latex'] node[above=2pt] {$\choice, \sig_2$} (q3);
			\draw (q2) edge[-latex'] node[above=2pt] {$\choice, \sig_2$} (q4);
			\draw (q3) edge[-latex',out=145,in=70] node[above=2pt] {$\sig_0$} (qinit);
			\draw (q4) edge[-latex',out=-145,in=-70] node[below=2pt] {$\sig_0$} (qinit);
			\draw (q1) edge[-latex',out=175,in=50] node[above=2pt] {$a$} (qmid);
			\draw (q2) edge[-latex'] node[above=1pt,xshift=2pt] {$a$} (qmid2);

			\draw (qmid) edge[-latex'] node[below=2pt,xshift=2pt] {$\sig_0, \frac{1}{2}$} (qinit);
			\draw (qmid) edge[-latex',out=0,in=-120] node[below=1pt,xshift=4pt] {$\sig_1, \frac{1}{4}$} (q1);
			\draw (qmid) edge[-latex'] node[right=6pt,yshift=-4pt] {$\sig_1, \frac{1}{4}$} (q2);
			\draw (qmid2) edge[-latex'] node[below=3pt,xshift=-2pt] {$\sig_0, \frac{1}{2}$} (qinit);
			\draw (qmid2) edge[-latex',out=-60,in=-180] node[below=4pt,xshift=-2pt] {$\sig_1, \frac{1}{2}$} (q2);
		\end{tikzpicture}
	\end{minipage}%
	\begin{minipage}{.5\textwidth}
		\centering
		\begin{tikzpicture}[every node/.style={font=\small,inner sep=1pt}]
			\draw ($(0,0)$) node[rond] (qinit) {$\set{\state_0}, 1$};
			\draw ($(qinit.west)-(0.5,0)$) edge[-latex'] (qinit);
			\draw ($(qinit)+(1.5,0)$) node[dot] (qmid1) {};
			\draw ($(qinit)+(3,0)$) node[rond,font=\tiny] (q1) {$\set{\state_1, \state_1'}, 1$};
			\draw ($(q1)+(2,0)$) node[rond,font=\tiny] (q3) {$\set{\state_2, \state_3}, 3$};
			\draw ($(q1)+(0,-1.5)$) node[dot] (qmid) {};

			\draw (qinit) edge[-latex'] (qmid1);
			\draw (qmid1) edge[-latex'] node[above=2pt] {$\frac{1}{2}$} (q1);
			\draw (qmid1) edge[-latex',out=120,in=30] node[above right] {$\frac{1}{2}$} (qinit);
			\draw (q1) edge[-latex'] node[above=2pt] {$\choice$} (q3);
			\draw (q1) edge[-latex'] node[left=2pt] {$a$} (qmid);
			\draw (qmid) edge[-latex',out=180,in=-60] node[below left] {$\frac{1}{2}$} (qinit);
			\draw (qmid) edge[-latex',out=0,in=-60] node[right] {$\frac{1}{2}$} (q1);
			\draw (q3) edge[-latex',out=120,in=60] (qinit);
		\end{tikzpicture}
	\end{minipage}
	\caption{The POMDP $\pomdp$ from~\Cref{ex:incompleteCE} (depicted on the left) with an almost-sure strategy, but whose belief-support MDP (depicted on the right) has no winning strategy.
		Notation $\state, k$ inside a circle depicts a state $\state$ with priority~$k$.
		Transitions from states involving a bullet $\bullet$ indicate a probabilistic transition.
		In POMDPs, we always write the signals along transitions.
		Actions are omitted when they all induce the same transition from a given state, and probabilities equal to $1$ are omitted.}
	\label{fig:incompleteCE}
\end{figure*}
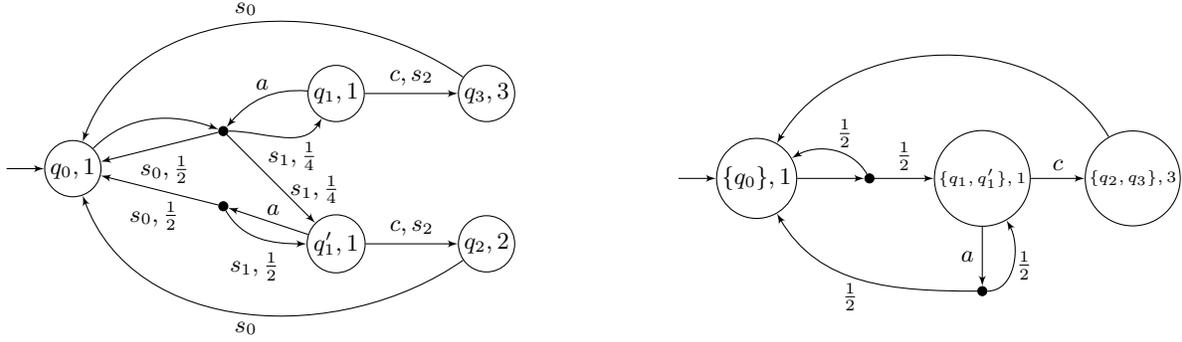

\begin{example} \label{ex:incompleteCE}
	Consider the POMDP $\pomdp$ in Figure~\ref{fig:incompleteCE}.
	This POMDP is weakly revealing, as state $\state_0$ is visited infinitely often for any strategy and is revealed through signal $\sig_0$.
	The only choice in this POMDP is in states $\state_1$ and $\state_1'$: whether to play~$a$ and move to $\state_0$ or $\set{\state_1, \state_1'}$, or to play $\choice$ and go to $\state_2$ or~$\state_3$.
	Observe that when the game starts in $\state_0$, the only reachable belief supports are $\set{\state_0}$, $\set{\state_1, \state_1'}$, and $\set{\state_2, \state_3}$, which all have a maximal odd priority.
	Hence, the belief-support MDP with priority function $\priMDP$ trivially has no almost-sure (and even positively) winning strategy.
	However, we show that there is an almost-sure strategy in $\pomdp$.

	The only way to win in this POMDP is to visit $\state_2$ infinitely often while visiting~$\state_3$ only finitely often.
	To do so, observe that when $a$ is played multiple times in a row and only receives signal $\sig_1$, the probability to be in $\state_1'$ becomes arbitrarily close to $1$.
	Formally, if $\strat_a$ is the strategy that only plays $a$, we have that for $n > 0$,
	\[
	\prob{\strat_a}{\pomdp}{\states^n\state_1'\mid (\sig_1)^n} =
	1 - \prob{\strat_a}{\pomdp}{\state_0(\state_1)^n\mid (\sig_1)^n} = 1 - \frac{1}{2^n}.
	\]

	For $n > 0$, let $\strat_{n}$ be the strategy that plays only $a$ until $\sig_1$ has been seen $n$ times \emph{in a row}, and when that is the case, plays $\choice$.
	Let us divide a play in this POMDP into rounds $1$, $2$,\ldots; every time we go back to $\state_0$ after visiting $\state_2$ or $\state_3$, we move to the next round.
	Consider the strategy that plays $\strat_{n}$ in round $n$.
	This strategy ensures that infinitely many rounds happen, because at each round $n$, it will eventually succeed in seeing $n$ occurrences of~$\sig_1$ in a row.
	At each round $n$, $\choice$ is eventually played with probability~$1$.
	By the above equation, $\state_3$ is seen with probability $\frac{1}{2^n}$ and $\state_2$ is seen with probability~$1 - \frac{1}{2^n}$.
	State $\state_2$ is clearly seen infinitely often almost surely, as the probability it is seen at each round is lower bounded by $\frac{1}{2}$.
	However, the probability that $\state_3$ is never seen anymore after round $n$ is equal to $\prod_{i = n}^\infty (1 - \frac{1}{2^i})$, which is positive and increases as $n$ grows to $\infty$.
	We deduce that the probability that $\state_3$ is seen at most finitely often is~$1$.

	Observe that no finite-memory strategy wins in this POMDP: such a strategy would need to play $\choice$ infinitely often, but could only do so after infixes of bounded length.
	Hence, the probability to reach $\state_3$ would be lower-bounded every time $\choice$ is played.
\end{example}

Generalizing the above example, we show that if we allow $\pri$ to take values in $\set{1, 2, 3}$, the existence of almost-sure strategies in weakly revealing POMDPs is undecidable.
We provide here a proof sketch; a full proof is in Appendix~\ref{app:weaklyRevealing}.

\begin{restatable}{theorem}{undecidableParity} \label{thm:undecidableParity}
	The existence of an almost-sure strategy in weakly revealing POMDPs with a parity objective with priorities in $\set{1, 2, 3}$ is undecidable.
	The same holds for the existence of a positively winning strategy.
\end{restatable}

Our proof uses a reduction from the value-$1$ problem in \emph{probabilistic automata}.
A \emph{probabilistic automaton}~\cite{Rabin:1963} is a tuple $\atmtn = (\states, \actions, \transitions, \initState)$.
One can define their semantics through POMDPs: they behave like POMDPs in which we assume that the signals bring no information ($\signals$ is a singleton).
No useful information is provided by the signals along a play (beyond the number of steps played); pure strategies therefore correspond to words on alphabet $\actions$.

Intuitively, the proof expands on the POMDP in Figure~\ref{fig:incompleteCE} by replacing states $\state_1$, $\state'_1$ by a copy of a probabilistic automaton $\atmtn$: the transition from $\initState$ goes to the initial state of $\atmtn$, and playing $\choice$ goes to $\state_2$ if the current state is a final state of $\atmtn$, and to $\state_3$ otherwise.
We keep a positive probability to go back to $\initState$ at any point to make it weakly revealing.
The idea of playing $n$ times $a$ in a row in the example is replaced by a (possible) sequence of words that have a probability arbitrarily close to~$1$ to reach a final state.
One can show that there is an almost-sure strategy in this POMDP if and only if $\atmtn$ has value~$1$ w.r.t.\ its final states.

\subsection{Decidability of the weakly revealing property} \label{subsec:weaklyRevealingDecidability}
To conclude this section, we discuss the complexity of deciding whether a POMDP is weakly revealing.
The property of weak revelations is defined with a B\"uchi condition about the \emph{belief supports}, which makes it different from the well-studied B\"uchi conditions on the \emph{states} of a POMDP.
Objectives related to reaching \emph{sets} of states rather than just states are sometimes called \emph{synchronization objectives}~\cite{DMS19,Doy23}.
However, to the best of our knowledge, they have not been studied for belief supports in POMDPs.

A direct argument shows that the weakly revealing property is decidable in $2$-\EXPTIME.
To see it, extend the POMDP with the information of the current belief support: this creates an exponential POMDP with state space $\states \times \powerSetNonEmpty{\states}$.
This extended POMDP has no positively winning strategy for $\coBuchi(\states\times\singletons^\pomdp)$ if and only if it is weakly revealing.
As the existence of a positively winning strategy for coB\"uchi objectives in POMDPs is \EXPTIME-complete~\cite{Chatterjee.Chmelik.ea:2016}, the complexity of this algorithm is doubly exponential.

We show that there is a better algorithm: deciding whether a POMDP is weakly revealing is \EXPTIME-complete.
The \EXPTIME-hardness is obtained by a reduction from the existence of a positive strategy for (state-based) safety objectives in POMDPs.
The membership in \EXPTIME is more complex: we study the complexity of the existence of a strategy that achieves $\Safety(B)$, where $B$ is a set of \emph{belief supports}, with a positive probability.
We show that this problem admits an \EXPTIME algorithm by reducing to an exponential-sized deterministic reachability game.
We then show that deciding whether a POMDP is weakly revealing can be reduced to polynomially many queries to this algorithm.
All details are in Appendix~\ref{app:weaklyRevealingProperty}; the result we obtain is the following.

\begin{restatable}{theorem}{complexityWeak} \label{thm:complexityWeak}
	Deciding whether a POMDP is weakly revealing is \textsf{EXPTIME}-complete.
\end{restatable}

\section{Strongly revealing POMDPs}
\label{sec:stronglyRevealing}
In this section, we introduce \emph{strongly revealing POMDPs}, a stronger property entailing that infinitely many revelations occur in a POMDP almost surely.
We show that the existence of almost-sure strategies is decidable for strongly revealing POMDPs with arbitrary parity objectives.

We define a notion of \emph{revealing signals}: for $\state$ a state of a POMDP $\pomdp = \pomdpFull$, we define $\revealing(\state) = \{\sig\in\signals \mid \forall r, r'\in\states, r'\neq \state \Longrightarrow \transitions(r, \act)(\sig, r') = 0\}$ to be the set of signals that indicate surely that the next state is $\state$.
For convenience, we define $\Succ(\state, \act) = \set{\state'\in\states \mid \exists \sig\in\signals,\transitions(\state, \act)(\sig, \state') > 0}$ and $\Succ(\state, \act, \sig) = \set{\state'\in\states \mid \transitions(\state, \act)(\sig, \state') > 0}$.

\begin{definition} \label{def:stronglyRevealing}
    POMDP $\pomdp = \pomdpFull$ is \emph{strongly revealing} if any transition between two states for a given action in the underlying MDP of $\pomdp$ can also happen with a revealing signal.
    Formally, $\pomdp$ is strongly revealing if for all $\state, \state'\in\states$ and $\act\in\actions$,
    if
       $\state'\in\Succ(\state, \act)$,
    then there is $\sig\in\revealing(\state')$ such that $\state'\in\Succ(\state, \act, \sig)$.
\end{definition}

Under this definition, the set of belief supports $\singletons^\pomdp$ is visited infinitely often from the initial state for any given strategy, so a strongly revealing POMDP is in particular weakly revealing.
Observe that the weakly revealing POMDP from Figure~\ref{fig:incompleteCE} is not strongly revealing: for instance, $\state_1'\in\Succ(\state_1, \act)$, but there is no revealing signal that could for sure reveal $\state_1'$ after $\state_1$.
The strongly revealing property can be decided in polynomial time in the size of a POMDP by simply analyzing every transition.

\begin{example} \label{ex:revealingTiger}
	We give an example of a strongly revealing POMDP inspired from the \emph{tiger} of~\cite{tiger94}, depicted in Figure~\ref{fig:tiger}.
	This example is the one used in Figure~\ref{fig:drl-vs-ours} in the introduction; the code to generate this example in our tool is also provided in Appendix~\ref{app:tiger}.

	In the tiger environment, an agent has to open the left or the right door, with action $\act_\mathsf{L}$ or $\act_\mathsf{R}$, respectively.
	One of them has a (deadly) tiger behind it. Fortunately, the agent can choose to wait and listen (action $\act_\mathsf{?}$) to help its decision.
	Listening results in a signal that is biased towards the reality, i.e., the
	signal can be $\sig_\mathsf{L}$ or $\sig_\mathsf{R}$ and the former is more likely if the tiger really is on the left, and vice versa.

	We present our version of the tiger environment in which listening guarantees one will eventually discern behind which door there is a tiger. This is
	achieved by adding new revealing signals $\act_\mathsf{L!}$ or $\act_\mathsf{R!}$ which,
	importantly, can only be obtained when the tiger is on the left or on the
	right, respectively. To keep things interesting, these signals can only
	be obtained with a small probability (yet, them being there already
	ensures that the POMDP is \emph{strongly revealing}). We also add
	signals for death ($\sig_\bot$) and victory ($\sig_\top$), which are missing from the original tiger environment.

	The addition of the signals $\sig_\mathsf{L!}$ and $\sig_\mathsf{R!}$ makes the objective easier to satisfy and therefore changes the semantics of the POMDP; indeed, it is possible to gather more information than in the original tiger environment. However, revealing the deadlock states $\top$ and $\bot$ through signals $\sig_\top$ and $\sig_\bot$ is necessary to make the POMDP strongly revealing but does not make the objective easier to satisfy, as no strategy can exit these states anyway.
\end{example}

\begin{figure}
	\centering
	\begin{tikzpicture}[every node/.style={font=\small,inner sep=1pt}]
		\draw (0,0) node[dot] (initDot) {};
		\draw ($(initDot.west)-(0.5,0)$) edge[-latex'] (initDot);
		\draw ($(initDot)+(1,1.)$) node[rond] (qL) {$\state_\mathsf{L}, 1$};
		\draw ($(initDot)+(1,-1.)$) node[rond] (qR) {$\state_\mathsf{R}, 1$};

		\draw (initDot) edge[-latex'] (qL);
		\draw (initDot) edge[-latex'] (qR);

		\draw ($(qR)+(-3,0)$) node[dot] (dotR) {};
		\draw (qR) edge[-latex'] node[above=2pt] {$\act_\mathsf{?}$} (dotR);
		\draw (dotR) edge[-latex',bend right] node[below=2pt] {$\sig_\mathsf{R}, .8\mid\sig_\mathsf{L}, .15\mid\sig_\mathsf{R!}, .05$} (qR);

		\draw ($(qL)+(-3,0)$) node[dot] (dotL) {};
		\draw (qL) edge[-latex'] node[below=2pt] {$\act_\mathsf{?}$} (dotL);
		\draw (dotL) edge[-latex',bend left] node[above=2pt] {$\sig_\mathsf{L}, .8\mid\sig_\mathsf{R}, .15\mid\sig_\mathsf{L!}, .05$} (qL);

		\draw ($(qL)+(2,0)$) node[rond] (top) {$\top, 2$};
		\draw ($(qR)+(2,0)$) node[rond] (bot) {$\bot, 1$};
		\draw (top) edge[-latex',loop right] node[right=2pt] {$\sig_\top$} (top);
		\draw (bot) edge[-latex',loop right] node[right=2pt] {$\sig_\bot$} (bot);

		\draw (qL) edge[-latex',bend right] node[above right,xshift=7pt,yshift=-5pt] {$a_\mathsf{L}, \sig_\bot$} (bot);
		\draw (qL) edge[-latex'] node[above=2pt] {$a_\mathsf{R}, \sig_\top$} (top);
		\draw (qR) edge[-latex',bend left] node[below right,xshift=7pt,yshift=5pt] {$a_\mathsf{L}, \sig_\top$} (top);
		\draw (qR) edge[-latex'] node[below=2pt] {$a_\mathsf{R}, \sig_\bot$} (bot);
	\end{tikzpicture}
	\caption{Strongly revealing \emph{tiger} (Example~\ref{ex:revealingTiger}).}
	\label{fig:tiger}
\end{figure}

\subsection{Decidability of parity with strong revelations}

The soundness of the analysis of the belief-support MDP for strongly revealing POMDPs follows from \Cref{prop:soundness}; it remains to show completeness (proof in Appendix~\ref{app:stronglyRevealing}).

\begin{restatable}{proposition}{strongComplete} \label{prop:strongComplete}
    Let $\pomdp = \pomdpFull$ be a strongly revealing POMDP with a parity objective specified by priority function $\pri$, and let $\beliefMDP$ be its belief-support MDP with priority function~$\priMDP$.
    If there is an almost-sure strategy for $\Parity(\pri)$ in $\pomdp$, then there is an almost-sure strategy for $\Parity(\priMDP)$ in $\beliefMDP$.
\end{restatable}

We also show a complexity lower bound.
The lower bound holds for CoB\"uchi objectives in strongly revealing POMDPs; as strongly revealing POMDPs are a subclass of weakly ones, the hardness follows for weakly revealing POMDPs.

\begin{restatable}{proposition}{exptimeHardnessRevealing} \label{prop:exptimeHardnessRevealing}
	The following problem is \EXPTIME-hard: given a strongly revealing POMDP with a CoB\"uchi objective, decide whether there is an almost-sure strategy.
\end{restatable}

We obtain as before the decidability of the problem by reducing to the analysis of the belief-support MDP.
The proof is similar to the one of Theorem~\ref{thm:decidableWeak}, simply replacing the use of \Cref{prop:completeness012} by \Cref{prop:strongComplete}.
\begin{theorem} \label{thm:stronglyDecidable}
    The existence of an almost-sure strategy for parity objectives in strongly revealing POMDPs is \EXPTIME-complete.
\end{theorem}
\begin{proof}
   Previous propositions show that the problem is in \EXPTIME: by Proposition~\ref{prop:soundness} (soundness of the belief-support MDP) and Proposition~\ref{prop:strongComplete} (completeness), we can reduce the problem to the existence of an almost-sure strategy for a parity objective in an MDP of size exponential in $\card{\states}$.
   The existence of an almost-sure strategy for parity objectives is decidable in polynomial time in MDPs~\cite[Theorem~10.127]{Baier.Katoen:2008}.

   The \EXPTIME-hardness follows from Proposition~\ref{prop:exptimeHardnessRevealing}, already for CoB\"uchi objectives.
\end{proof}

\subsection{Undecidability of CoB\"uchi games with strong revelations} \label{sec:games}
We study here whether the revealing semantics helps in \emph{zero-sum games} of partial information with revealing semantics.
In general, such games with CoB\"uchi objectives are undecidable (they encompass POMDPs) while B\"uchi games are decidable for almost-sure strategies~\cite{Bertrand.Genest.ea:2017}.
We show a negative result: the existence of an almost-sure strategy in two-player CoB\"uchi \emph{games} with partial information is undecidable, even when restricted to games satisfying a natural extension of the strongly revealing property.

Two-player games of partial information are tuples $\game = \gameFull$, where $\states$ is a finite set of states, $\actions_1$ and $\actions_2$ are finite sets of actions, $\signals$ is a finite set of signals, $\transitions\colon \states\times\actions_1\times\actions_2\to\dist{\signals\times\states}$ is the transition function, and $\initState\in\states$ is an initial state.
At each round, two players Player~1 and Player~2 respectively choose an action in $\actions_1$ and $\actions_2$.
Histories and plays are defined as for POMDPs.
For $i\in\{1, 2\}$, a strategy of Player~$i$ is a function $\strat_i\colon (\actions_i \times \signals)^* \to \dist{\actions_i}$.
Given two strategies $\strat_1$ and $\strat_2$ for the two players, we can define as for POMDPs a probability measure $\prob{\strat_1, \strat_2}{\pomdp}{\cdot}$ on plays.
A strategy $\strat_1$ of Player~$1$ is \emph{almost sure} for an objective $\objective$ if for all strategies $\strat_2$ of Player~2, $\prob{\strat_1, \strat_2}{\pomdp}{\objective} = 1$.
To express CoB\"uchi objectives, we consider as before a priority function $\pri\colon\states\to \{0, 1\}$.

A game $\game$ is \emph{strongly revealing} if for any possible transition between two states with a pair of actions, there is a chance of a signal that reveals the second state; the definition is the same as for POMDPs, assuming $\actions = \actions_1 \times \actions_2$.

\begin{restatable}{theorem}{undecidableGames} \label{thm:undecidableGames}
    The existence of an almost-sure strategy for \Pone
    in strongly revealing CoB\"uchi games is undecidable.
\end{restatable}

\subsection{Optimistic semantics for POMDPs}
\label{sec:optimistic_semantics}
In our revealing definitions, we adopted the point of view of considering \emph{subclasses} of POMDPs.
A limitation of this point of view is that our results say nothing about POMDPs which are not strongly (nor weakly) revealing.
We argue that another fruitful formulation of our results concerns the class of \emph{all} POMDPs, by defining alternative, \emph{revealing} semantics.

Consider a POMDP $\pomdp$.
Let us define the extended POMDP $\pomdp_\mathsf{sr}$ such that, at each transition, there is a small probability of revealing which state we reach after firing this action, using additional signals $\sig_\state$, one for each state $\state$ of $\pomdp$.

\begin{theorem}
	For any POMDP $\pomdp$, $\pomdp_\mathsf{sr}$ is strongly revealing.
	Moreover, if there is no almost-sure strategy ensuring an omega-regular objective in $\pomdp_\mathsf{sr}$ (which is decidable by Theorem~\ref{thm:stronglyDecidable}), then there is no almost-sure strategy ensuring the same objective in $\pomdp$.
\end{theorem}

The contrapositive is easily proved: any almost-sure strategy of $\pomdp$ can be lifted to an almost-sure strategy of $\pomdp_\mathsf{sr}$.
This property justifies the term ``optimistic semantics''. Note that the converse implication cannot hold (as POMDPs with omega-regular objectives are undecidable).

We compare our approach with a well-known, even more optimistic semantics: revealing the exact state at each transition, which corresponds to working on the underlying MDP.
The following example shows that studying the strongly revealing semantics is finer than studying the underlying MDP (i.e., it proves the non-existence of almost-sure strategies for more POMDPs).

\begin{example} \label{ex:optimisticSemantics}
	Consider the POMDP $\pomdp$ with a CoBüchi objective depicted in Figure~\ref{fig:optimisticSemantics}.
	In this POMDP, the only way to win is to get to state $\top$; for this, it is necessary to play $a$ from $\state_a$ or $b$ from $\state_b$.
	However, this is only possible by knowing the exact state ($\state_a$ or $\state_b$) after one move.
	Therefore, there is an almost-sure strategy in the underlying MDP, but not in the strongly revealing $\pomdp_\mathsf{sr}$.
\end{example}

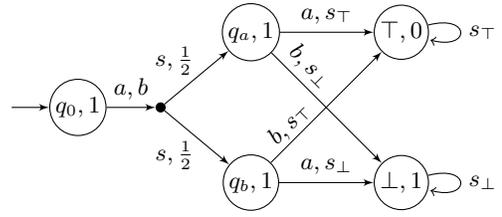
\begin{figure}
	\centering
	\begin{tikzpicture}[every node/.style={font=\small,inner sep=1pt}]
		\draw ($(0,0)$) node[rond] (qinit) {$\state_0, 1$};
		\draw ($(qinit.west)-(0.5,0)$) edge[-latex'] (qinit);
		\draw ($(qinit)+(1.1,0)$) node[dot] (initDot) {};
		\draw ($(initDot)+(1.2,1.)$) node[rond] (qa) {$\state_a, 1$};
		\draw ($(initDot)+(1.2,-1.)$) node[rond] (qb) {$\state_b, 1$};

		\draw (qinit) edge[-latex'] node[above=2pt] {$a, b$} (initDot);
		\draw (initDot) edge[-latex'] node[above left] {$\sig, \frac{1}{2}$} (qa);
		\draw (initDot) edge[-latex'] node[below left] {$\sig, \frac{1}{2}$} (qb);

		\draw ($(qa)+(2,0)$) node[rond] (top) {$\top, 0$};
		\draw ($(qb)+(2,0)$) node[rond] (bot) {$\bot, 1$};

		\draw (qa) edge[-latex'] node[above=2pt] {$a, \sig_\top$} (top);
		\draw (qa) edge[-latex'] node[above,xshift=-10pt,yshift=12pt,rotate=-45] {$b, \sig_\bot$} (bot);
		\draw (qb) edge[-latex'] node[above=2pt,xshift=-10pt,yshift=-12pt,rotate=45] {$b, \sig_\top$} (top);
		\draw (qb) edge[-latex'] node[above=2pt] {$a, \sig_\bot$} (bot);
		\draw (top) edge[-latex',loop right] node[right=2pt] {$\sig_\top$} (top);
		\draw (bot) edge[-latex',loop right] node[right=2pt] {$\sig_\bot$} (bot);
	\end{tikzpicture}
	\caption{POMDP $\pomdp$ with a CoB\"uchi objective such that there is an almost-sure strategy in the underlying MDP, but not in the strongly revealing $\pomdp_\mathsf{sr}$ (Example~\ref{ex:optimisticSemantics}).}
	\label{fig:optimisticSemantics}
\end{figure}

The fact that our approach is finer than considering the underlying MDP could already be guessed from the computational complexity of the problems: solving MDPs with parity objectives is in $\mathsf{P}$~\cite{Baier.Katoen:2008}, while solving strongly revealing POMDPs with parity objectives is $\EXPTIME$-complete (Theorem~\ref{thm:stronglyDecidable}).
If presented with a POMDP with a parity objective beyond a Büchi one, we therefore suggest to try to solve the underlying MDP in polynomial time; if there is no almost-sure strategy, then this is also the case for the original POMDP; otherwise, try to solve the strongly revealing POMDP in exponential time; if there is no almost-sure strategy, then this is also the case for the original POMDP.

\section{Related works}
\label{sec:relatedWorks}
We discuss additional references where a restriction is set to stochastic systems to make them decidable.

The closest idea to our revelations that we know of is in~\cite{Berwanger.Mathew:2017}, defining a class of partial-information multi-player games with \emph{sure} (not just almost-sure) revelations; from any point in the game, a ``revelation'' occurs surely within a bounded number of steps.
This is a yet stronger kind of revelation mechanism under which even parity \emph{games} are decidable.

The \emph{decisiveness property}~\cite{Abdulla.Ben-Henda.ea:2007,Bertrand.Bouyer.ea:2020} is a useful property to decide reachability properties in infinite stochastic systems (without decision-making).
Decisiveness is implied by the existence of a \emph{finite attractor}; there is such an attractor in weakly revealing POMDPs once we fix a finite-memory strategy (as in Proposition~\ref{prop:soundness}).
An extension was considered for infinite MDPs~\cite{Bertrand.Bouyer.ea:2020} to approximate the value of reachability, which is orthogonal to our objectives.

Another path to decidability and strong guarantees is to restrict strategies, such as studying ``memoryless''~\cite{VLB12} or finite-memory~\cite{Chatterjee.Chmelik.ea:2016,ACJK22} strategies in POMDPs.
In our paper, the strategies we consider only use finite memory, as they are memoryless strategies on the belief-support MDP (in our case, they are even shown to be optimal among all strategies under the right assumptions).
The sufficiency of belief-support-based strategies in POMDPs, which was known for almost-sure reachability~\cite{Baier.Groer.ea:2012}, was also exploited to craft efficient algorithms in~\cite{JJS20}; such an approach could speed up our algorithms.

An interesting class of decidable POMDPs with parity objectives is given by the \emph{multi-environment MDPs}~\cite{Raskin.Sankur:2014}, which are POMDPs consisting of multiple copies of the same MDP, where only the transition function changes. The only partial observation comes from not knowing in which copy we play; this is therefore also a restriction on the information loss. This class is incomparable to our classes of revealing POMDPs.

In a quantitative setting, the idea of having actions with some cost (time or energy) that reveal the current state or decrease the uncertainty has appeared multiple times in the literature. Such an idea already appeared in 2011~\cite{BG11} for POMDPs with quantitative reachability objectives. Recently, \emph{active-measuring POMDPs}, with a similar mechanism, have been considered in the online planning community~\cite{Bellinger.Coles.ea:2021,Krale.Simao.ea:2023}.
Despite a different setting (online planning vs.\ model-checking), it carries an intuition similar to our work: precise states can be known, which helps find good strategies.

Also in online planning, the article~\cite{Liu.Chung.ea:2022} considers a subclass of POMDPs restricting information loss that make \emph{reinforcement learning} sample efficient.

\section{Perspectives} \label{sec:perspectives}
We presented classes of POMDPs for which many natural objectives become decidable, and showed that these lie close to undecidability frontiers (priorities $\set{0, 1, 2}$ vs.\ $\set{1, 2, 3}$, POMDPs vs.\ games).

Due to their intrinsic undecidability, POMDPs are not often studied through the prism of exact algorithms.
We believe there is a lot to gain by understanding more closely $(i)$ the \emph{structural properties} of POMDPs that make them decidable for classes of objectives (such as weak/strong revelations), and $(ii)$ the conditions that make \emph{simple strategies} (such as belief-support-based strategies) sufficient.
Our article is a new step towards these goals.

On a more specific note, an interesting step for $(i)$ could involve framing the exact complexity of the existence of strategies for simple objectives involving beliefs and belief supports, as was started in the proof of Theorem~\ref{thm:complexityWeak}.

\section*{Acknowledgments}
This work was partially supported by the SAIF project, funded by the ``France 2030'' government investment plan managed by the French National Research Agency, under the reference ANR-23-PEIA-0006.
Pierre~Vandenhove was funded by ANR project G4S (ANR-21-CE48-0010-01).
This work was sparked by discussions with Guillaume Vigeral and Bruno Ziliotto, following a talk on a related model~\cite{VZ22}.

\bibliography{revealing_pomdps}

\newpage
\appendix
\onecolumn
\large

\section{Revealing Tiger Environment} \label{app:tiger}
We provide the code to generate the \emph{strongly revealing tiger} POMDP, used in Figure~\ref{fig:drl-vs-ours} in the introduction, and described in Example~\ref{ex:revealingTiger} and Figure~\ref{fig:tiger}.

For concreteness, we provide the environment encoded in the
\emph{Cassandra} format for POMDPs.
States \texttt{tiger-left}, \texttt{tiger-right}, \texttt{dead}, and \texttt{done} correspond respectively to $\state_\mathsf{L}$, $\state_\mathsf{R}$, $\state_\bot$, and $\state_\top$.
Actions \texttt{listen}, \texttt{open-left}, and \texttt{open-right} correspond respectively to $\act_\mathsf{?}$, $\act_\mathsf{L}$, and $\act_\mathsf{R}$.
Observations \texttt{maybe-left} and \texttt{maybe-right} correspond to signals $\sig_\mathsf{L}$ and $\sig_\mathsf{R}$, \texttt{defo-left} and \texttt{defo-right} correspond to $\sig_\mathsf{L!}$ and $\sig_\mathsf{R!}$, and \texttt{dead-obs} and \texttt{done-obs} correspond to $\sig_\bot$ and $\sig_\top$.
As a basis for comparison, the original
(discounted) tiger environment is given here:
\url{https://www.pomdp.org/examples/tiger.aaai.POMDP}.

\begin{minted}{cfs}
states: tiger-left tiger-right dead done
actions: listen open-left open-right
observations: maybe-left maybe-right defo-left defo-right dead-obs done-obs
start include: tiger-left tiger-right

T:listen
identity

T:open-left
0.00 0.00 1.00 0.00
0.00 0.00 0.00 1.00
0.00 0.00 1.00 0.00
0.00 0.00 0.00 1.00

T:open-right
0.00 0.00 0.00 1.00
0.00 0.00 1.00 0.00
0.00 0.00 1.00 0.00
0.00 0.00 0.00 1.00

O:listen
0.80 0.15 0.05 0.00 0.00 0.00
0.15 0.80 0.00 0.05 0.00 0.00
0.00 0.00 0.00 0.00 1.00 0.00
0.00 0.00 0.00 0.00 0.00 1.00

O:open-left
0.00 0.00 0.00 0.00 1.00 0.00
0.00 0.00 0.00 0.00 0.00 1.00
0.00 0.00 0.00 0.00 1.00 0.00
0.00 0.00 0.00 0.00 0.00 1.00

O:open-right
0.00 0.00 0.00 0.00 0.00 1.00
0.00 0.00 0.00 0.00 1.00 0.00
0.00 0.00 0.00 0.00 1.00 0.00
0.00 0.00 0.00 0.00 0.00 1.00
\end{minted}

\newpage
\section{Additional preliminaries: End components in MDPs} \label{app:endComponents}
In proofs, we will use extensively the notion of \emph{end components} for various MDPs derived from POMDPs. End components were introduced in~\cite{Alfaro:1997} and have been widely used in the literature on MDPs; see the book~\cite[Section~10.6.3]{Baier.Katoen:2008} for a more thorough introduction.

Let $\mdp = \mdpFull$ be an MDP.
For a pair $\ec = \ecFull$ with $\ecStates \subseteq \states$ and $\ecActions\colon \ecStates\to 2^\actions$, we define the graph \emph{induced by $\ecFull$} as the graph $(\ecStates, \edges)$, where $\edges$ is the set of edges $(\state, \state')\in\ecStates\times\ecStates$ for which there is $\act\in\ecActions(\state)$ such that $\transitions(\state, \act)(\state') > 0$.
An \emph{end component} of $\mdp$ is a tuple $\ec = \ecFull$, with $\ecStates \subseteq \states$ and $\ecActions\colon \ecStates\to 2^\actions$, such that for all $\state\in\ecStates$ and $\act\in\ecActions(\state)$, $\supp{\transitions(\state, \act)} \subseteq \ecStates$, and such that the graph induced by $\ecFull$ is strongly connected.

For $\play \in (\actions\cdot\states)^\omega$ a play of $\mdp$, we write $\mathsf{inf}(\play)$ for the pair $(\ecStates, \ecActions)$ such that $\ecStates$ is the set of states that $\play$ visits infinitely often and such that for $\state\in\ecStates$, $\ecActions(\state)$ is the set of actions played infinitely often when in $\state$ through~$\play$.
The main result we need about end components is the following.

\begin{theorem}[Fundamental theorem of end components~\cite{Alfaro:1997}] \label{thm:ec}
	Let $\mdp = \mdpFull$ be an MDP.
	\begin{itemize}
		\item If $\ec = \ecFull$ is an end component of $\mdp$ and $\state\in\ecStates$, there is a strategy $\strat$ such that $\prob{\strat}{\mdp^\state}{\mathsf{inf}(\play) = \ec} = 1$.
		\item For all strategies $\strat$, $\prob{\strat}{\mdp}{\mathsf{inf}(\play)\ \text{is an end component}} = 1$.
	\end{itemize}
\end{theorem}

\section{Properties of the belief-support MDP (appendix to Section~\ref{sec:beliefSupportMDP})} \label{app:beliefSupportMDP}

In this section, we provide full statements and proofs for claims in Section~\ref{sec:beliefSupportMDP}.

Observe that there is a syntactic difference in what strategies of a POMDP and strategies in the belief-support MDP can observe: strategies in the belief-support MDP can see the belief supports, but do not know through which signal they were generated (there may be multiple such signals).
Strategies in a POMDP are therefore syntactically more general: they can observe the precise sequence of signals, and thereby compute the belief supports visited.
Despite this difference between the power of strategies in both models, both models behave similarly for multiple simple objectives.

First, there is a correspondence between the sequences of belief supports reached positively in both models.

\begin{lemma} \label{lem:bijectionStrategies}
	Let $\pomdp = \pomdpFull$ be a POMDP and $\beliefMDP$ be its belief-support MDP.
	\begin{itemize}
		\item For all strategies~$\strat\in\strats{\pomdp}$, there is a strategy $\strat_\beliefUpd\in\strats{\beliefMDP}$ such that, for all sequences of belief supports $B\in (\powerSetNonEmpty{\states})^*$,
		\[
		\prob{\strat}{\pomdp}{\Cyl(B)} > 0 \Longleftrightarrow \prob{\strat_\beliefUpd}{\beliefMDP}{\Cyl(B)} > 0.
		\]

		\item For all strategies~$\strat_\beliefUpd\in\strats{\beliefMDP}$, for all sequences of belief supports $B\in (\powerSetNonEmpty{\states})^*$,
		\[
		\prob{\liftStrat}{\pomdp}{\Cyl(B)} > 0 \Longleftrightarrow \prob{\strat_\beliefUpd}{\beliefMDP}{\Cyl(B)} > 0.
		\]
	\end{itemize}
\end{lemma}
\begin{proof}
	We write $\beliefMDP = \beliefMDPFull$.

	We show the first claim.
	Let $\strat\in\strats{\pomdp}$.
	We want to define a strategy $\strat_\beliefUpd\in\strats{\beliefMDP}$ that imitates what $\strat$ does, but such a strategy cannot observe the signals.
	Let $\hist_\beliefUpd = \act_1\beliefSupp_1\ldots\act_n\beliefSupp_n$ be a history of $\beliefMDP$.
	We define $\sigEvents_{\hist_\beliefUpd}$ as the set $\{\hist\in(\actions\cdot\signals)^* \mid \text{$\hist$ is consistent with $\strat$ and}\ B_\hist = \hist_\beliefUpd\}$.

	We define $\strat_\beliefUpd$ inductively on the length of histories, and show at the same time that $\sigEvents_{\hist_\beliefUpd}$ is non-empty for all histories $\hist_\beliefUpd$ consistent with $\strat_\beliefUpd$.
	After $0$ step, there is a single possible empty history $\hist_\beliefUpd$ consistent with $\strat_\beliefUpd$.
	The set $\sigEvents_{\hist_\beliefUpd}$ contains only the empty history and is indeed non-empty.

	Now, let $\hist_\beliefUpd = \act_1\beliefSupp_1\ldots\act_n\beliefSupp_n$ be a history of length~$n$ consistent with $\strat_\beliefUpd$.
	By induction hypothesis, $\sigEvents_{\hist_\beliefUpd}$ is non-empty.
	We use this fact to define $\strat_\beliefUpd(\hist_\beliefUpd)$.
	Let $A_{\hist_\beliefUpd} = \{\act\in\actions \mid \exists\hist \in \sigEvents_{\hist_\beliefUpd}, \strat(\hist)(\act) > 0\}$, which is also non-empty.
	We define $\strat_\beliefUpd(\hist_\beliefUpd)$ to randomize over all actions in $A_{\hist_\beliefUpd}$.
	Let $\act_{n+1}\in A_{\hist_\beliefUpd}$ and $\beliefSupp_{n+1}$ be such that $\transitionsMDP(\beliefSupp_n, \act_{n+1})(\beliefSupp_{n+1}) > 0$.
	By induction hypothesis, there is $\hist\in \sigEvents_{\hist_\beliefUpd}$ of length $n$ consistent with $\strat$ such that $\strat(\hist)(\act_{n+1}) > 0$.
	Hence, by construction of the belief-support MDP, $\sigEvents_{\hist\act_{n+1}\beliefSupp_{n+1}}$ is non-empty.

	Let $B = \beliefSupp_1\ldots\beliefSupp_n\in (\powerSetNonEmpty{\states})^*$ be a sequence of belief supports.
	We now show that
	\[
	\prob{\strat}{\pomdp}{\Cyl(B)} > 0 \Longleftrightarrow \prob{\strat_\beliefUpd}{\beliefMDP}{\Cyl(B)} > 0.
	\]
	If $\prob{\strat_\beliefUpd}{\beliefMDP}{\Cyl(B)} > 0$, then there is $\hist_\beliefUpd$ consistent with $\strat_\beliefUpd$ such that $\hist_\beliefUpd = \act_1\beliefSupp_1\ldots\act_n\beliefSupp_n$ for some $\act_1,\ldots,\act_n\in\actions$.
	As $\sigEvents_{\hist_\beliefUpd}$ is non-empty, we also have $\prob{\strat}{\pomdp}{\Cyl(B)} > 0$.
	If~$\prob{\strat}{\pomdp}{\Cyl(B)} > 0$, then there is $\hist\in\sigEvents_{\hist_\beliefUpd}$ consistent with $\strat$ such that $B_\hist = \act_1\beliefSupp_1\ldots\act_n\beliefSupp_n$ for some $\act_1,\ldots,\act_n\in\actions$.
	One can show by induction on the length of $\hist$ that~$B_\hist$ is also consistent with $\strat_\beliefUpd$, so $\prob{\strat_\beliefUpd}{\beliefMDP}{\Cyl(B)} > 0$.

	We now show the second claim.
	Let $\strat_\beliefUpd\in\strats{\beliefMDP}$; consider $\liftStrat\in\strats{\pomdp}$ (which was defined in Section~\ref{sec:beliefSupportMDP}).
	Clearly, the tree of belief supports induced by both strategies contains the same nodes (although they are attained with possibly different probabilities).
	Again, this can be shown with a similar (and easier) induction on the length of histories.
	Hence, both strategies see the same cylinders of belief supports with positive probability.
\end{proof}

Second, sets of belief supports reached almost surely in the POMDP can also be reached almost surely in the belief-support MDP.

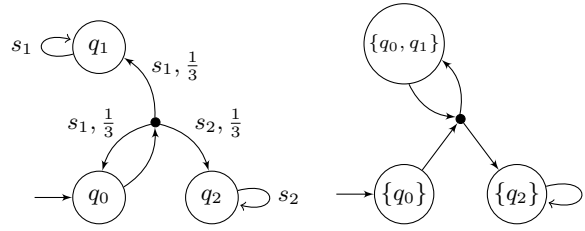
\begin{figure}
	\centering
	\begin{tikzpicture}[every node/.style={font=\small,inner sep=1pt}]
		\draw ($(0,0)$) node[rond] (q0) {$\state_0$};
		\draw ($(q0.west)-(0.5,0)$) edge[-latex'] (q0);
		\draw ($(q0)+(0.75,1)$) node[dot] (q0dot) {};
		\draw ($(q0)+(0,2)$) node[rond] (q1) {$\state_1$};
		\draw ($(q0)+(1.5,0)$) node[rond] (q2) {$\state_2$};

		\draw (q0) edge[-latex',bend right] (q0dot);
		\draw (q1) edge[-latex',loop left] node[left=2pt] {$\sig_1$} (q1);
		\draw (q0dot) edge[-latex',bend right] node[above left] {$\sig_1, \frac{1}{3}$} (q0);
		\draw (q0dot) edge[-latex',bend right] node[above right] {$\sig_1, \frac{1}{3}$} (q1);
		\draw (q0dot) edge[-latex',bend left] node[above right] {$\sig_2, \frac{1}{3}$} (q2);

		\draw (q2) edge[-latex',out=60,in=120,loop right] node[right=2pt] {$\sig_2$} (q2);
	\end{tikzpicture}\hspace{1em}
	\begin{tikzpicture}[every node/.style={font=\small,inner sep=1pt}]
		\draw ($(0,0)$) node[rond] (q0) {$\set{\state_0}$};
		\draw ($(q0.west)-(0.5,0)$) edge[-latex'] (q0);
		\draw ($(q0)+(0.75,1)$) node[dot] (q0dot) {};
		\draw ($(q0)+(0,2)$) node[rond,font=\scriptsize] (q1) {$\set{\state_0, \state_1}$};
		\draw ($(q0)+(1.5,0)$) node[rond] (q2) {$\set{\state_2}$};

		\draw (q0) edge[-latex'] (q0dot);
		\draw (q1) edge[-latex',bend right] (q0dot);
		\draw (q0dot) edge[-latex',bend right] (q1);
		\draw (q0dot) edge[-latex'] (q2);

		\draw (q2) edge[-latex',out=60,in=120,loop right] (q2);
	\end{tikzpicture}%
	\caption{POMDP with a single action (left) and its belief-support MDP (right) such that the belief support $\{\state_2\}$ is reached almost surely in the belief-support MDP, but not in the POMDP.}
	\label{fig:ASnotPreserved}
\end{figure}

\begin{lemma} \label{lem:bijectionStrategiesAS}
	Let $\pomdp = \pomdpFull$ be a POMDP and $\beliefMDP$ be its belief-support MDP.
	Let~$B\subseteq\powerSetNonEmpty{\states}$ be a set of belief supports.

	\begin{itemize}
		\item If there is $\strat\in\strats{\pomdp}$ such that $\prob{\strat}{\pomdp}{\Reach(B)} = 1$, then there is $\strat_\beliefUpd\in\strats{\beliefMDP}$ such that $\prob{\strat_\beliefUpd}{\beliefMDP}{\Reach(B)} = 1$.

		\item If for all $\strat\in\strats{\pomdp}$ it holds that $\prob{\strat}{\pomdp}{\Reach(B)} = 1$, then for all $\strat_\beliefUpd\in\strats{\beliefMDP}$ we have $\prob{\strat_\beliefUpd}{\beliefMDP}{\Reach(B)} = 1$.
	\end{itemize}
\end{lemma}

The converses of both claims of Lemma~\ref{lem:bijectionStrategiesAS} are false in general; see Figure~\ref{fig:ASnotPreserved}.
The proof uses the notion of \emph{end components}, which is discussed in Appendix~\ref{app:endComponents}.
\begin{proof}
	We first prove the first item.
	Let $\strat\in\strats{\pomdp}$ such that $\prob{\strat}{\pomdp}{\Reach(B)} = 1$.
	Without loss of generality, we modify $\strat$ and $\pomdp$ such that $\strat$ moves to a sink state $\bot$ with a new action $\act_\bot$ and a new signal $\sig_\bot$ whenever a belief support in $B$ is seen.
	Let $B_\strat = \{\beliefSupp\in\powerSetNonEmpty{\states} \mid \beliefSupp \neq \{\bot\}\ \text{and}\ \prob{\strat}{\pomdp}{\Reach(\beliefSupp)} > 0\}$.
	Any belief support $\beliefSupp\in B_\strat$ is such that there is an almost-sure strategy $\strat^\beliefSupp$ from $\beliefSupp$ for $\Reach(B)$ that only visits belief supports in $B_\strat$ (simply take $\strat^\beliefSupp$ to be a continuation of $\strat$ after $\beliefSupp$ is visited).
	We build a strategy $\strat_\beliefUpd\in\strats{\beliefMDP}$ such that $\prob{\strat_\beliefUpd}{\beliefMDP}{\Reach(B)} = 1$.

	By Lemma~\ref{lem:bijectionStrategies}, for all $\beliefSupp\in B_\strat$, there is a strategy $\strat_\beliefUpd^\beliefSupp\in\strats{\beliefMDP}$ that reaches with a positive probability exactly the same belief supports as $\strat^\beliefSupp$.
	In particular, $\prob{\strat_\beliefUpd^\beliefSupp}{\beliefMDP^\beliefSupp}{\Reach(\beliefSupp')} > 0$ implies that $\beliefSupp'\in B_\strat$.
	Moreover, $\prob{\strat_\beliefUpd^\beliefSupp}{\beliefMDP^\beliefSupp}{\Reach(B)} > 0$.
	By the continuity of probabilities and the finiteness of $B_\strat$, there is a uniform $k\in\IN$ such that $\prob{\strat_\beliefUpd^\beliefSupp}{\beliefMDP^\beliefSupp}{\Reach^{\le k}(B)} > 0$ for all~$\beliefSupp$.
	Moreover, as $B_\strat$ is finite, there is $\alpha > 0$ such that $\prob{\strat_\beliefUpd^\beliefSupp}{\beliefMDP^\beliefSupp}{\Reach^{\le k}(B)} \ge \alpha$ for all~$\beliefSupp\in B_\strat$.

	We define the strategy $\strat_\beliefUpd$ as follows: we start by following
	$\strat^{\{\initState\}}_\beliefUpd$ for $k$ steps, which gives a probability $\ge \alpha$ to reach~$B$.
	If $B$ is not reached within $k$ steps, we know that the current belief support $\beliefSupp$ is still in $B_\strat$.
	We then move on to the strategy $\strat^{\beliefSupp}_\beliefUpd$ for $k$ steps and iterate this process.
	Strategy $\strat_\beliefUpd$ has a lower-bounded probability to reach $B$ infinitely many times, and therefore reaches $B$ almost surely.

	We now show the contrapositive of the second item.
	Assume that there exists $\strat_\beliefUpd\in\strats{\beliefMDP}$ such that $\prob{\strat_\beliefUpd}{\beliefMDP}{\Reach(B)} < 1$.
	By Theorem~\ref{thm:ec}, it implies that there is an end component $\ec = \ecFull$ of $\beliefMDP$ such that $\ecStates \cap B = \emptyset$ and $\prob{\strat_\beliefUpd}{\beliefMDP}{\play \notin \Reach(B)\ \text{and}\ \mathsf{inf}(\play) = \ec} > 0$.
	By Lemma~\ref{lem:bijectionStrategies}, the strategy $\liftStrat$ therefore also has a non-zero probability to reach a belief support in $\ec$ without reaching $B$ and to then never visit a belief outside of $\ec$.
	Hence, $\prob{\liftStrat}{\pomdp}{\Reach(B)} < 1$.
\end{proof}

\section{Probabilistic bounds on reachability objectives} \label{app:boundsReachability}
We prove technical lemmas about reachability properties that we will use multiple times in the paper: roughly, if a set of states (or a set of belief supports) is reached with a positive probability by \emph{all} strategies, we can upper bound the number of steps before reaching it and lower bound the probability to reach it.

We will use the result that pure strategies suffice in POMDPs for all objectives, both for almost-sure strategies and for attaining values, which we recall from~\cite{Chatterjee.Doyen.ea:2010}.
\begin{theorem}[{\citet[Theorem~5]{Chatterjee.Doyen.ea:2010}}] \label{thm:pureSuffice}
	Let $\pomdp = \pomdpFull$ be a POMDP and $\objective \subseteq \states^\omega$ be an objective.
	We have that $\sup_{\strat\in\strats{\pomdp}} \prob{\strat}{\pomdp}{\objective} = \sup_{\strat\in\stratsPure{\pomdp}} \prob{\strat}{\pomdp}{\objective}$.
	Moreover, if there exists $\strat\in\strats{\pomdp}$ such that $\prob{\strat}{\pomdp}{\objective} = \sup_{\strat\in\strats{\pomdp}} \prob{\strat}{\pomdp}{\objective}$, then there exists $\strat'\in\stratsPure{\pomdp}$ such that $\prob{\strat'}{\pomdp}{\objective} = \sup_{\strat\in\strats{\pomdp}} \prob{\strat}{\pomdp}{\objective}$.
\end{theorem}

Before proving Lemma~\ref{lem:pomdpReach} about POMDPs, we prove a similar (and easier) result about MDPs.

\begin{lemma} \label{lem:mdpReach}
	Let $\mdp = \mdpFull$ be an MDP and $\reach\subseteq \states$.
	Assume that for all $\strat\in\strats{\mdp}$, $\prob{\strat}{\mdp}{\Reach(\reach)} > 0$.
	Then, for all $\strat\in\strats{\mdp}$, $\prob{\strat}{\mdp}{\Reach^{\le \card{\states}}(\reach)} \ge \leastProb^{\card{\states}}_\mdp$.
\end{lemma}
\begin{proof}
	We consider a deterministic variant of the MDP $\mdp$ where an antagonistic opponent resolves the stochastic transitions, by choosing among any transition that had a non-zero probability in the original MDP.
	Formally, we consider the two-player deterministic game $\game = (\states, \states \times \actions, \edges, \initState)$ where the set of edges $\edges = \set{(\state, (\state, \act))\mid\state\in\states, \act\in\actions} \cup \set{((\state, \act), \state') \mid \transitions(\state, \act)(\state') > 0}$: \Pone chooses the action in states of $\states$, and \Ptwo chooses the transitions in states of $\states\times\actions$.

	\Pone has no winning strategy for $\Safety(\reach)$ in this deterministic zero-sum game.
	We show the contrapositive: assume that \Pone has a winning strategy for $\Safety(\reach)$ in $\game$.
	Then, this strategy would be surely winning in $\game$, and thus with probability~$1$, in $\mdp$ for $\Safety(\reach)$.
	So there exists a strategy $\strat\in\strats{\mdp}$ such that $\prob{\strat}{\mdp}{\Reach(\reach)} = 0$.
	This contradicts the hypothesis.

	Therefore, \Pone has no winning strategy for $\Safety(\reach)$ in $\game$.
	As reachability games are determined, this means that \Ptwo has a strategy that ensures $\Reach(\reach)$.
	Classical results on reachability games (using the \emph{attractor decomposition}~\cite{Bloem.Chatterjee.ea:2018}) imply that \Ptwo even has such a strategy ensuring that no state is visited twice along the path before reaching a state in $\reach$.
	Hence, \Ptwo has a strategy ensuring that \Pone plays at most $\card{\states}$ times before visiting $\reach$.

	This means that, in the original MDP, no matter the strategy, there is a path from $\initState$ to~$\reach$ of length at most $\card{\states}$.
	For pure strategies, there are at most $\card{\states}$ probabilistic transitions along this path, so such a path has probability at least~$\leastProb_\mdp^{\card{\states}}$.
	As pure strategies suffice for reachability objectives in MDPs, this result extends to all strategies.
\end{proof}

We can show such a result for POMDPs, both for sets of states and sets of belief supports.
To make sense of this statement, we extend reachability objectives to deal with (sets of) belief supports.
For $B\subseteq \powerSetNonEmpty{\states}$, let $\sigEvents_B \subseteq \signals^*$ be the set of observable histories $\act_1\sig_1\ldots\act_n\sig_n$ such that $\beliefUpd^*(\set{\initState}, \act_1\sig_1\ldots\act_n\sig_n) \in B$.
We define $\Reach(B)$ (resp.\ $\Buchi(B)$) to be the set of plays inducing a belief support in $B$ at least once (resp.\ infinitely often).
For $\beliefSupp\in\powerSetNonEmpty{\states}$, we write $\Reach(\beliefSupp)$ and $\Buchi(\beliefSupp)$ for $\Reach(\set{\beliefSupp})$ and $\Buchi(\set{\beliefSupp})$.
Note that the proof of the following Lemma~\ref{lem:pomdpReach} uses properties of the belief-support MDP $\beliefMDP$ proved in Appendix~\ref{app:beliefSupportMDP}.

\begin{lemma} \label{lem:pomdpReach}
	Let $\pomdp = \pomdpFull$ be a POMDP.
	\begin{itemize}
		\item Let $B\subseteq \powerSetNonEmpty{\states}$ be a set of belief supports.
		If for all $\strat\in\strats{\pomdp}$, $\prob{\strat}{\pomdp}{\Reach(B)} > 0$,
		then for all $\strat\in\strats{\pomdp}$, $\prob{\strat}{\pomdp}{\Reach^{\le {2^{\card{\states}}-1}}(B)} \ge \leastProb^{2^{\card{\states}}-1}_\pomdp$.
		\item Let $\reach\subseteq \states$ be a set of states.
		If for all $\strat\in\strats{\pomdp}$, $\prob{\strat}{\pomdp}{\Reach(\reach)} > 0$,
		then for all $\strat\in\strats{\pomdp}$, $\prob{\strat}{\pomdp}{\Reach^{\le {2^{\card{\states}}-1}}(\reach)} \ge \leastProb^{2^{\card{\states}}-1}_\pomdp$.
	\end{itemize}
\end{lemma}
\begin{proof}
	We first prove the claim about belief supports.
	Let $B\subseteq \powerSetNonEmpty{\states}$.
	Assume that for all $\strat\in\strats{\pomdp}$, $\prob{\strat}{\pomdp}{\Reach(B)} > 0$.
	Then, by Lemma~\ref{lem:bijectionStrategies}, for all $\strat_\beliefUpd\in\strats{\beliefMDP}$, $\prob{\strat_\beliefUpd}{\beliefMDP}{\Reach(B)} > 0$.
	As we now work with an MDP with $\card{\powerSetNonEmpty{\states}} = 2^{\card{\states}}-1$ states, we can use Lemma~\ref{lem:mdpReach} and obtain that for all $\strat_\beliefUpd\in\strats{\beliefMDP}$, $\prob{\strat_\beliefUpd}{\beliefMDP}{\Reach^{\le {2^{\card{\states}}-1}}(B)} > 0$.
	Going back to $\pomdp$ with Lemma~\ref{lem:bijectionStrategies}, we have that for all $\strat\in\strats{\pomdp}$, $\prob{\strat}{\pomdp}{\Reach^{\le {2^{\card{\states}}-1}}(B)} > 0$.
	Hence, for all strategies, there is a path of length at most $2^{\card{\states}}-1$ reaching $B$.
	For pure strategies, such a path has probability at least $\leastProb_\pomdp^{{2^{\card{\states}}-1}}$ to happen, as there are at most $2^{\card{\states}}-1$ random transitions along this path.
	Hence, for all \emph{pure} strategies $\strat\in\stratsPure{\pomdp}$, $\prob{\strat}{\pomdp}{\Reach^{\le {2^{\card{\states}}-1}}(B)} \ge \leastProb_\pomdp^{{2^{\card{\states}}-1}}$.
	By Theorem~\ref{thm:pureSuffice}, we deduce that for all (even non-pure) strategies $\strat\in\strats{\pomdp}$, $\prob{\strat}{\pomdp}{\Reach^{\le {2^{\card{\states}}-1}}(B)} \ge \leastProb_\pomdp^{{2^{\card{\states}}-1}}$.

	For the second claim, let $\reach\subseteq \states$ be a set of states.
	Assume that for all $\strat\in\strats{\pomdp}$, $\prob{\strat}{\pomdp}{\Reach(\reach)} > 0$.
	Let $B_\reach = \{\beliefSupp\in\powerSetNonEmpty{\states} \mid \beliefSupp\cap\reach \neq \emptyset \}$ be the set of belief supports containing a state in $\reach$.
	As the event ``visiting $\reach$'' implies the event ``visiting a belief support in~$B_\reach$'', we also have that for all $\strat\in\strats{\pomdp}$, $\prob{\strat}{\pomdp}{\Reach(B_\reach)} > 0$.
	By the first property, for all strategies $\strat\in\strats{\pomdp}$, $\prob{\strat}{\pomdp}{\Reach^{\le {2^{\card{\states}}-1}}(B_\reach)} > 0$.
	This means once again that for all strategies, there is a path of length at most $2^{\card{\states}}-1$ that reaches $B_\reach$, so that reaches~$\reach$.
	We conclude in the same way as the first property:
	for all \emph{pure} strategies $\strat\in\stratsPure{\pomdp}$, $\prob{\strat}{\pomdp}{\Reach^{\le {2^{\card{\states}}-1}}(\reach)} \ge \leastProb_\pomdp^{{2^{\card{\states}}-1}}$; so, by Theorem~\ref{thm:pureSuffice}, for all strategies $\strat\in\strats{\pomdp}$, $\prob{\strat}{\pomdp}{\Reach^{\le {2^{\card{\states}}-1}}(\reach)} \ge \leastProb_\pomdp^{2^{\card{\states}}-1}$.
\end{proof}

As a corollary, we obtain that in POMDPs,
if a safety objective has value~$1$, then there is actually an almost-sure strategy for the safety objective.
This contrasts with reachability objectives, for which there are POMDPs that have value~$1$ but no almost-sure strategy.
For reachability, the computational complexity of the two problems are widely different: the existence of an almost-sure strategy is EXPTIME-complete~\cite{Baier.Bertrand.ea:2008}, but deciding value~$1$ is undecidable~\cite{Gimbert.Oualhadj:2010}.

\begin{corollary} \label{cor:value1Safety}
	Let $\pomdp = \pomdpFull$ be a POMDP and $\safe \subseteq \states$ be a set of states.
	We have that
	\begin{equation*}
		\exists \strat\in\strats{\pomdp}, \prob{\strat}{\pomdp}{\Safety(\safe)} = 1
		\Longleftrightarrow
		\sup_{\strat\in\strats{\pomdp}} \prob{\strat}{\pomdp}{\Safety(\safe)} = 1.
	\end{equation*}
\end{corollary}

\begin{proof}
	The implication $\Longrightarrow$ is trivial; we focus on the other implication.
	We prove the contrapositive.
	Assume that there is no almost-sure strategy, i.e., that for all strategies $\strat$, we have $\prob{\strat}{\pomdp}{\Safety(\safe)} < 1$.
	In other words, this means that for all strategies $\strat\in\strats{\pomdp}$, we have $\prob{\strat}{\pomdp}{\Reach(\safe)} > 0$.
	By Lemma~\ref{lem:pomdpReach}, there is thus $\alpha > 0$ such that for all strategies $\strat\in\strats{\pomdp}$, we have $\prob{\strat}{\pomdp}{\Reach^{\le 2^{\card{\states}}-1}(\safe)} \ge \alpha$.
	As $\prob{\strat}{\pomdp}{\Reach(\safe)} \ge \prob{\strat}{\pomdp}{\Reach^{\le 2^{\card{\states}}-1}(\safe)}$, we have that for all strategies $\strat\in\strats{\pomdp}$, $\prob{\strat}{\pomdp}{\Reach(\safe)} \ge \alpha$.
	Hence, for all strategies $\strat\in\strats{\pomdp}$, $\prob{\strat}{\pomdp}{\Safety(\safe)} \le 1 - \alpha$, so $\sup_{\strat\in\strats{\pomdp}} \prob{\strat}{\pomdp}{\Safety(\safe)} < 1$.
\end{proof}

\section{Additional details for Section~\ref{sec:weaklyRevealing}~: (un)decidability of parity objectives in weakly revealing POMDPs}
\label{app:weaklyRevealing}

This section is devoted to the proofs of statements about weakly revealing POMDPs with parity objectives from Section~\ref{sec:weaklyRevealing}.
We restate and prove the soundness of the analysis of the belief-support MDP.

\soundness*
\begin{proof}
	As $\strat_\beliefUpd$ is pure and memoryless, we assume it is a function $\powerSetNonEmpty{\states} \to \actions$.
	Consider strategy $\liftStrat$: it is a pure strategy in $\pomdp$ with exponentially many memory states, as it simply looks at the current belief support and plays as $\strat_\beliefUpd$ would.
	By Lemma~\ref{lem:bijectionStrategies}, $\liftStrat$ and $\strat_\beliefUpd$ reach the same belief supports with a positive probability.

	We first prove a correspondence between the sets of belief supports reached almost surely by $\liftStrat$ and $\strat_\beliefUpd$, in the other direction than the one of Lemma~\ref{lem:bijectionStrategiesAS}.
	Using that $\pomdp$ is weakly revealing and that the strategies only use finite memory, we show that for all $B\subseteq \powerSetNonEmpty{\states}$,
	\begin{align} \label{eq:almostSureMemoryless}
		\prob{\strat_\beliefUpd}{\beliefMDP}{\Buchi(B)} = 1 \Longrightarrow \prob{\liftStrat}{\pomdp}{\Buchi(B)} = 1.
	\end{align}
	Let $B\subseteq \powerSetNonEmpty{\states}$ be such that $\prob{\strat_\beliefUpd}{\beliefMDP}{\Buchi(B)} = 1$.
	Let $\states_\infty = \{\state\in\states \mid \prob{\liftStrat}{\pomdp}{\Buchi(\{\state\})} > 0\}$.
	Due to $\pomdp$ being weakly revealing, $\states_\infty$ is non-empty and $\prob{\liftStrat}{\pomdp}{\Buchi(\states_\infty)} = 1$.
	Every $\{\state\}$ with $\state\in\states_\infty$ is reached positively by~$\strat_\beliefUpd$ (Lemma~\ref{lem:bijectionStrategies}), so $\liftStrat$ still reaches an element of $B$ with non-zero probability from these singleton beliefs.
	Note that whenever a $\{\state\}$ with $\state\in\states_\infty$ is reached by $\liftStrat$, the probability to reach $B$ afterwards is independent from the past: this is due to belief support $\{\state\}$ describing precisely the current belief (the probability to be in $\state$ is~$1$) and $\liftStrat$ being memoryless w.r.t.\ belief supports.
	Hence, infinitely many visits to singleton belief supports $\{\state\}$ with $\state\in\states_\infty$ are done by $\liftStrat$, and the probability to reach $B$ afterwards is lower-bounded by a uniform value.
	We conclude that $\prob{\strat}{\pomdp}{\Buchi(B)} = 1$.

	We now consider the set of end components $\ECs_{\strat_\beliefUpd} = \{U = (\ecStates_U, \ecActions_U) \mid \prob{\strat_\beliefUpd}{\beliefMDP}{\inf(\play) = U} > 0\}$ that strategy~$\strat_\beliefUpd$ can exactly end up in in $\beliefMDP$.
	Set $\ECs_{\strat_\beliefUpd}$ is non-empty by Theorem~\ref{thm:ec}.
	Each such end component $\ec\in\ECs_{\strat_\beliefUpd}$
	\begin{itemize}
		\item contains a belief singleton $\set{\state_U}$ for some $\state_U\in\states$: otherwise, using that $\pomdp$ is weakly revealing, this contradicts Lemma~\ref{lem:bijectionStrategiesAS};
		\item has an even maximal priority w.r.t.\ $\priMDP$, due to $\strat_\beliefUpd$ being almost sure for $\Parity(\priMDP)$ and $U$ being an end component of $\strat_\beliefUpd$ with positive probability.
	\end{itemize}
	The second property implies that each end component $U\in\ECs_{\strat_\beliefUpd}$ is such that $\max \set{\pri(\state) \mid \state\in\beliefSupp\in \ecStates_\ec}$ is even. Let $\state_U^{\max}$ be a state in some belief support of $\ec$ achieving this maximal priority in $U$.

	Let us consider the effect of the strategy $\liftStrat$ in $\pomdp$.
	First, observe that this strategy almost surely reaches a belief support $\set{\state}$ for some $U\in\ECs_{\strat_\beliefUpd}$ and $\state\in\ecStates_U$, which follows from~\eqref{eq:almostSureMemoryless}.

	We fix an end component $\ec = (\ecStates_\ec, \ecActions_\ec)\in\ECs_{\strat_\beliefUpd}$.
	Assume $\liftStrat$ has reached a belief support $\{\state\}$ with $\state\in\ecStates_\ec$.
	Using again~\eqref{eq:almostSureMemoryless}, we know that such singleton belief supports are almost surely visited infinitely often.
	Following strategy~$\liftStrat$, by definition of the belief-support MDP, there is a non-zero probability to reach~$\state_U^{\max}$ from every singleton belief support $\set{\state}$ with $\state\in\ecStates_\ec$.
	Due to $\liftStrat$ being memoryless w.r.t.\ belief supports, this probability is the same after each visit to such a $\set{\state}$.
	Hence, state~$\state_U^{\max}$ has infinitely often a lower-bounded probability to be visited.
	As these probabilities are independent,~$\state_U^{\max}$ is visited infinitely often.
	To conclude, observe that once a belief support in $U$ is reached, $\liftStrat$ only visits states in the belief supports of~$U$ and visits $\state_U^{\max}$ infinitely often.
	The strategy $\liftStrat$ is therefore almost sure for $\Parity(\pri)$ in~$\pomdp$.
\end{proof}

We now show an example illustrating that revelations may take exponentially many steps to occur with positive probability in weakly revealing POMDPs.
More generally, this example also shows that bounds for belief supports in Lemma~\ref{lem:pomdpReach} are tight, even for weakly revealing POMDPs.

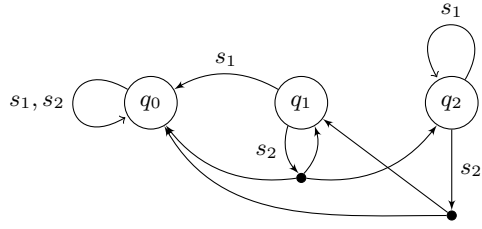
\begin{figure}[t]
	\centering
	\begin{tikzpicture}[every node/.style={font=\small,inner sep=1pt}]
		\draw ($(0,0)$) node[rond] (q0) {$\state_0$};
		\draw ($(q0)+(2,0)$) node[rond] (q1) {$\state_1$};
		\draw ($(q1)+(0,-1)$) node[dot] (q1s2) {};
		\draw ($(q1)+(2,0)$) node[rond] (q2) {$\state_2$};
		\draw ($(q2)+(0,-1.5)$) node[dot] (q2s2) {};

		\draw (q0) edge[-latex',out=150,in=210,loop] node[left=2pt] {$\sig_1, \sig_2$} (q0);
		\draw (q1) edge[-latex',out=150,in=30] node[above=2pt] {$\sig_1$} (q0);
		\draw (q1) edge[-latex',bend right] node[left=2pt] {$\sig_2$} (q1s2);
		\draw (q1s2) edge[-latex',bend left] (q0);
		\draw (q1s2) edge[-latex',bend right] (q1);
		\draw (q1s2) edge[-latex',bend right] (q2);
		\draw (q2) edge[-latex',out=60,in=120,loop] node[above=2pt] {$\sig_1$} (q2);
		\draw (q2) edge[-latex'] node[right=2pt] {$\sig_2$} (q2s2);
		\draw (q2s2) edge[-latex',out=180,in=-60] (q0);
		\draw (q2s2) edge[-latex'] (q1);
	\end{tikzpicture}%
	\begin{tikzpicture}[every node/.style={font=\small,inner sep=1pt}]

	\end{tikzpicture}
	\caption{The POMDP from \Cref{ex:expLowerBound} (with $n = 2$): a partially observable Markov chain which is weakly revealing, and such that a revelation takes at best exponentially many steps from the initial state $\state_\mathsf{init}$.
		The initial state $\state_\mathsf{init}$ is not represented, it simply moves to any other state in a non-observable way.
		Precise probabilities are omitted.
		One transition of each type appears: when the signal has a lower index than the state (e.g., $\sig_1$ from $\state_2$), the probability mass stays in the state; when it matches the state (e.g., $\sig_2$ from $\state_2$), the probability mass is spread over the smaller states; when it is greater (e.g., $\sig_2$ from $\state_1$), the probability mass is spread over all states.
		The shortest revealing sequence from $\state_\mathsf{init}$ is $\set{\state_\mathsf{init}} \to \set{\state_0, \state_1, \state_2} \xrightarrow{\sig_1} \set{\state_0, \state_2} \xrightarrow{\sig_2} \set{\state_0, \state_1} \xrightarrow{\sig_1} \set{\state_0}$.}
	\label{fig:expLowerBound}
\end{figure}

\begin{example} \label{ex:expLowerBound}
	Let $\pomdp = (\states, \actions, \signals, \transitions, \state_\mathsf{init})$ be a POMDP with $\states = \{\state_\mathsf{init}, \state_0, \ldots, \state_n\}$, $\actions = \{a\}$, and $\signals = \{\sig_1, \ldots, \sig_n\}$.
	We depict the construction for $n = 2$ in Figure~\ref{fig:expLowerBound}.
	As $\card{\actions} = 1$, it is actually a \emph{partially observable Markov chain}.
	The transition function is defined as follows:
	\begin{itemize}
		\item $\transitions(\state_\mathsf{init}, a)(\sig, \state_i) > 0$ for all $\sig\in\signals$, $i\in\{0, \ldots, n\}$,
		\item $\transitions(\state_i, a)(\sig_j, \state_k) > 0$ if and only if ($i = k = 0$) or ($i > j$ and $i = k$) or ($i = j$ and $i > k$) or ($i < j$).
	\end{itemize}

	Here is how a play happens: first, the probability mass is spread from $\state_\mathsf{init}$ to all the other states.
	The state $\state_0$ is absorbing: there will always be some probability mass in state $\state_0$, so a revelation can only happen in this state.
	Here is how the other states with index $i\in\set{1, \ldots, n}$ behave:
	\begin{itemize}
		\item if a signal $\sig_j$ with index $j < i$ is seen, then any probability mass already in $\state_i$ simply remains in $\state_i$;
		\item if the signal $\sig_i$ is seen, then any probability mass in $\state_i$ is spread out over the states $\state_k$ with $k < i$.
		\item if a signal $\sig_j$ with index $j > i$ is seen, then any probability mass in $\state_i$ is spread out over all states in $\set{\state_0, \ldots, \state_n}$.
	\end{itemize}

	To reason on this POMDP, we associate every belief support with a number that we read in binary on the belief support: we assume that a belief support $\beliefSupp\in\powerSetNonEmpty{\states}$ is associated with number $f(\beliefSupp) = \sum_{\state_i\in\beliefSupp, i\ge 1} 2^{i-1}$ (in particular, $f(\set{\state_0}) = 0$).
	The initial belief (after one step) is number $2^n - 1 = f(\set{\state_0, \ldots, \state_n})$.
	Let us study how the number $f(\beliefSupp)$ evolves given the signals seen:
	\begin{itemize}
		\item for $j = \min \set{i\ge 1\mid \state_i\in\beliefSupp}$, we have $f(\beliefUpd(\beliefSupp, \act\sig_j)) = f(\beliefSupp) - 1$: indeed, state $\state_j$ in the belief support becomes $\{\state_0, \ldots, \state_{j-1}\}$, while other states remain the same;
		\item for $j < \min \set{i\ge 1\mid \state_i\in\beliefSupp}$, we have $f(\beliefUpd(\beliefSupp, \act\sig_j)) = f(\beliefSupp)$: the belief support remains exactly the same;
		\item for $j > \min \set{i\ge 1\mid \state_i\in\beliefSupp}$, we have $f(\beliefUpd(\beliefSupp, \act\sig_j)) = 2^n - 1$, as a state $\state_i$ with $i < j$ is such that $\beliefUpd(\set{\state_i}, \act\sig_j) = \set{\state_0, \ldots, \state_n}$.
	\end{itemize}

	Hence, the shortest path towards a revelation from belief support $\set{\state_0, \ldots, \state_n}$ has length $2^n - 1$.

	As mentioned in Remark~\ref{rem:memoryLowerBound}, one can show an exponential lower bound on the memory that strategies need to play almost surely in weakly revealing POMDPs.
	To do so, we slightly modify the POMDP $\pomdp$ used above.
	We assume all states have priority~$1$, and add two sink states $\state_\top$ and $\state_\bot$ with priorities respectively $0$ and $1$.
	Add a second action $\choice$ such that $\choice$ goes to $\state_\top$ from state $\state_0$ and to $\state_\bot$ from any other state.
	There is an almost-sure strategy for the CoB\"uchi objective: play action $a$ until it is certain that the current state is $\state_0$, and then play $\choice$.
	This strategy has an exponential size, and by the above analysis, any smaller strategy cannot ensure that $\state_\top$ is reached almost surely.
\end{example}

We note that the soundness of the belief-support MDP holds for the CoB\"uchi objective for general POMDPs, without any revealing hypothesis.

\begin{lemma} \label{lem:soundForCoBuchi}
	Let $\pomdp = \pomdpFull$ be a POMDP with a CoB\"uchi objective specified by priority function $\pri\colon \states\to\set{0,1}$, and let $\beliefMDP$ be its belief-support MDP with priority function~$\priMDP$.
	If there is an almost-sure strategy for $\Parity(\priMDP)$ in $\beliefMDP$, then there is an almost-sure strategy for $\Parity(\pri)$ in $\pomdp$.
\end{lemma}
\begin{proof}
	The proof carries out as the one of \Cref{prop:soundness}, also using strategy $\liftStrat$.
	However, the analysis of end components is simpler: simply notice that no end component in $\ECs_{\strat_\beliefUpd}$ can contain a belief support with priority~$1$.
	In particular, strategy $\liftStrat$ eventually does not encounter any state with priority~$1$.
\end{proof}

We now restate and prove the completeness of the belief-support MDP for weakly revealing POMDPs with priorities restricted to $\set{0, 1, 2}$.
\completenessZeroOneTwo*
Before giving the proof, we compare the technique to the proof of soundness of belief-support MDPs for weakly revealing POMDPs (\Cref{prop:soundness}).
Two elements make this proof less straightforward:
\begin{itemize}
	\item we start from an almost-sure strategy in a POMDP, on which we cannot exploit memory bounds (as opposed to the case of MDPs, for which we know that pure memoryless strategies suffice for parity objectives); in general, infinite-memory strategies are required in POMDPs~\cite[Theorem~2]{Chatterjee.Doyen.ea:2013};
	\item we cannot just play a ``copy'' of the winning strategy of the POMDP in the MDP.
	In general, a strategy may be almost sure in a POMDP with a CoB\"uchi objective, while visiting infinitely many belief supports containing a state with priority~$1$.
	For instance, in the POMDP from Example~\ref{ex:incompleteCE}, we saw that there is a strategy that visits $\state_3$ finitely often and $\state_2$ infinitely often, which means that the belief $\{\state_2, \state_3\}$ is still visited infinitely often.
	Starting from an arbitrary winning strategy $\strat$ in the POMDP, we may need to modify it to win in the MDP.
\end{itemize}
\begin{proof}[Proof of \Cref{prop:completeness012}]
	Let $\strat\in\strats{\pomdp}$ be an almost-sure strategy for $\Parity(\pri)$ in~$\pomdp$ (we recall that $\pri$ takes values in $\set{0, 1, 2}$).
	Let $\states_\infty = \set{\state\in\states \mid \prob{\strat}{\pomdp}{\Buchi(\set{\state})} > 0}$.
	Due to $\pomdp$ being weakly revealing, the set $\states_\infty$ is non-empty.
	Moreover, $\prob{\strat}{\pomdp}{\bigcup_{\state\in\states_\infty}\Buchi(\set{\state})} = 1$.

	The strategy we build in $\beliefMDP$ first tries to reach a belief support $\{\state\}$ for some $\state\in\states_\infty$.
	By Lemma~\ref{lem:bijectionStrategiesAS}, there is a strategy that achieves this with probability~$1$.
	We then show that if we reach a state $\{\state\}$ for some $\state\in\states_\infty$, then we can continue with an almost-sure strategy in~$\beliefMDP$ for $\Parity(\priMDP)$.
	We distinguish two cases.

	Assume first there is a strategy from $\state\in\states_\infty$ in $\pomdp$ that satisfies $\Safety(\pri^{-1}(1))$ almost surely.
	In terms of belief supports, this is equivalent to satisfying $\Safety(B_1)$ almost surely, where $B_1 = \{\beliefSupp\in\powerSetNonEmpty{\states}\mid \exists \state\in\beliefSupp, \pri(\state) = 1\}$.
	By Lemma~\ref{lem:bijectionStrategies}, there is therefore a strategy from $\{\state\}$ in $\beliefMDP$ that avoids any belief in $B_1$ almost surely, which is winning almost surely.

	We now assume that there is no almost-sure strategy for $\Safety(\pri^{-1}(1))$ from~$\state$.
	By Corollary~\ref{cor:value1Safety}, this implies that this safety objective does not have value~$1$.
	Therefore, a state with priority~$1$ has a lower-bounded probability to be visited after each visit to~$\state$: there is $\alpha > 0$ such that for all strategies $\strat'\in\strats{\pomdp}$, $\prob{\strat'}{\pomdp^{\state}}{\Reach(\pri^{-1}(1))} \ge \alpha$.
	From this, we deduce that $\prob{\strat}{\pomdp}{\Buchi(\pri^{-1}(1)) \mid \Buchi(\{\state\})} = 1$.
	Since $\prob{\strat}{\pomdp}{\Buchi(\{\state\})} > 0$ and $\prob{\strat}{\pomdp}{\Parity(\pri) \mid \Buchi(\{\state\})} =~1$, we conclude that $\prob{\strat}{\pomdp}{\Buchi(\pri^{-1}(2)) \mid \Buchi(\{\state\})} =~1$.
	Deducing from this a property over belief supports, if we take $B_2 = \{\beliefSupp\in\powerSetNonEmpty{\states}\mid \exists \state\in\beliefSupp, \pri(\state) = 2\}$, we have $\prob{\strat}{\pomdp}{\Buchi(B_2) \mid \Buchi(\{\state\})} = 1$.
	We deduce that in $\beliefMDP$, there must be an end component containing both $\{\state\}$ and a belief support in $B_2$.
	By Theorem~\ref{thm:ec}, there is an almost-sure strategy from $\{\state\}$ in $\beliefMDP$ for $\Parity(\priMDP)$, which ends the proof.
\end{proof}

We now move on the proof of undecidability of parity objectives with priorities $1$, $2$, and $3$ for weakly revealing POMDPs.

\undecidableParity*

A \emph{probabilistic automaton}~\cite{Rabin:1963} is a tuple $\atmtn = (\states, \actions, \transitions, \initState)$.
One can define their semantics through POMDPs: they behave like POMDPs in which we assume that the signals bring no information ($\signals$ is a singleton).
No useful information is provided by the signals along a play (beyond the number of steps played); pure strategies therefore correspond to words on alphabet $\actions$.

We define the \emph{value-$1$ problem} for probabilistic automata, which has been shown to be undecidable~\cite{Gimbert.Oualhadj:2010,Fijalkow:2017}.
We will reduce from this problem to prove Theorem~\ref{thm:undecidableParity}.
Let $\atmtn = (\states, \actions, \transitions, \initState)$ be a probabilistic automaton and $\finalStates \subseteq \states$ be a set of ``final'' states.
For $\strat\in A^*$, we denote by~$\belief^\atmtn_\strat$ the belief after playing $\strat$.
We write $\belief^\atmtn_\strat(\finalStates) = \sum_{\state\in\finalStates} \belief^\atmtn_\strat(\state)$.
The value-$1$ problem is the following: given a probabilistic automaton $\atmtn = (\states, \actions, \transitions, \initState)$ and a set of states $\finalStates\subseteq\states$, do we have $\sup_{\strat\in A^*} \belief^\atmtn_\strat(\finalStates) = 1$?

\begin{proof}[Proof of Theorem~\ref{thm:undecidableParity}]
	Let $\atmtn = \atmtnFull$ be a probabilistic automaton, and $\finalStates\subseteq \states^\atmtn$.
	To interpret $\atmtn$ as a POMDP, we assume that all actions receive a single signal $\sig^\atmtn$.
	We consider a generalization of the POMDP of Example~\ref{ex:incompleteCE}: roughly, we replace the two states $\state_1$ and $\state_1'$ by a copy of $\atmtn$, where $\state_1'$ plays the role of states of $\finalStates$.

	\begin{figure}[t]
		\centering
		\begin{tikzpicture}[every node/.style={font=\small,inner sep=1pt}]
			\draw ($(0,0)$) node[rond] (qinit) {$\initState, 1$};
			\draw ($(qinit.west)-(0.5,0)$) edge[-latex'] (qinit);

			\draw ($(qinit)+(2,2)$) node[rond] (q1) {$\initState^\atmtn, 1$};
			\draw ($(qinit)+(2,-2)$) node[rond,accepting] (q2) {};
			\draw ($(q1)+(.8,-2)$) node[rond] (q1') {};
			\node[rectangle, fit=(q1) (q2) (q1'), draw=black,inner sep=8pt] (A) {};
			\draw ($(q1')+(1,0)$) node[] {$\atmtn$};

			\draw ($(q1)+(4,0)$) node[rond] (q3) {$\state_3, 3$};
			\draw ($(q2)+(4,0)$) node[rond] (q4) {$\state_2, 2$};

			\draw ($(q1)+(-0.5,-1.5)$) node[dot] (qmid) {};
			\draw (qinit) edge[-latex'] node[above left] {$\sig_1$} (q1);
			\draw (q1) edge[-latex'] node[above=2pt] {$\choice, \sig_2$} (q3);
			\draw (q1') edge[-latex'] node[above left] {$\choice, \sig_2$} (q3);
			\draw (q2) edge[-latex'] node[above=2pt] {$\choice, \sig_2$} (q4);
			\draw (q3) edge[-latex',out=150,in=90,looseness=1.4] node[above=2pt] {$\sig_0$} (qinit);
			\draw (q4) edge[-latex',out=-150,in=-90,looseness=1.4] node[below=2pt] {$\sig_0$} (qinit);
			\draw (q1) edge[-latex'] node[left=2pt] {} (qmid);

			\draw (qmid) edge[-latex'] node[below=2pt] {$\sig_0, \frac{1}{2}$} (qinit);
			\draw (qmid) edge[-latex'] node[below=2pt,xshift=-1pt] {} (q1');
		\end{tikzpicture}
		\caption{POMDP $\pomdp^\atmtn$ used in the proof of Theorem~\ref{thm:undecidableParity}.
			The rectangle contains a copy of probabilistic automaton $\atmtn$, with all transitions having probability $\frac{1}{2}$ to go back to $\initState$.
			When playing $\choice$ from a state of $\atmtn$, either $\state_3$ is reached if the state is not in~$\finalStates$, or $\state_2$ is reached if the state is in $\finalStates$ (represented by the double circle).
			This POMDP has an almost-sure strategy for the parity objective if and only if $\atmtn$ has value~$1$ w.r.t.\ $\finalStates$.}
		\label{fig:undecidable}
	\end{figure}
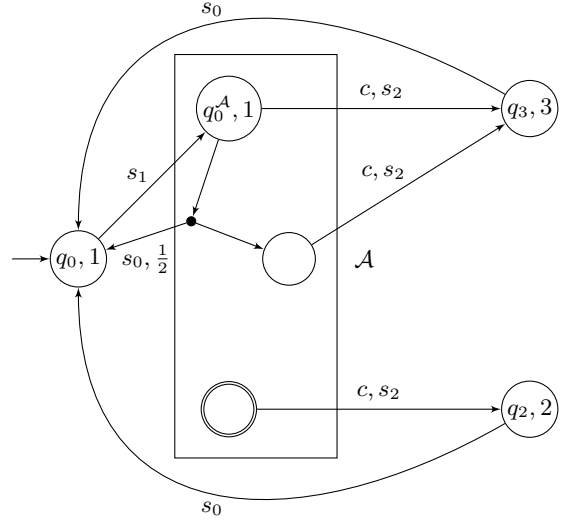

	Formally, we consider the POMDP $\pomdp^\atmtn = \pomdpFull$ (depicted in Figure~\ref{fig:undecidable}) such that
	\begin{itemize}
		\item $\states = \set{\state_0, \state_2, \state_3} \disjUnion \states^\atmtn$,
		\item $\actions = \set{\choice} \disjUnion \actions^\atmtn$,
		\item $\signals = \set{\sig_0, \sig_1, \sig_2, \sig^\atmtn}$,
		\item for all $\act\in\actions$, $\transitions(\initState, \act)(\sig_1, \initState^\atmtn) = \transitions(\state_2, \act)(\sig_0, \initState) = \transitions(\state_3, \act)(\sig_0, \initState) = 1$,
		\item for all $\state,\state'\in\states^\atmtn$, $\act\in\actions^\atmtn$,
		$\transitions(\state, \act)(\sig^\atmtn, \state') = \frac{\transitions_\atmtn(\state, \act)(\state')}{2}$, $\transitions(\state, \act)(\sig_0, \initState) = \frac{1}{2}$,
		\item for $\state\in\states^\atmtn \setminus \finalStates$, $\transitions(\state, \choice)(\sig_2, \state_3) = 1$, and for $\state\in\finalStates$, $\transitions(\state, \choice)(\sig_2, \state_2) = 1$.
	\end{itemize}

	POMDP $\pomdp^\atmtn$ is weakly revealing: all strategies visit $\state_0$ infinitely often almost surely, and all visits to $\state_0$ are revealed through signal $\sig_0$.
	We define a priority function $\pri$ in $\pomdp^\atmtn$ with values in $\{1, 2, 3\}$: $\pri(\state_2) = 2$, $\pri(\state_3) = 3$, and $\pri(\state) = 1$ for all $\state\in\states^\atmtn\cup\{\initState\}$.

	The following claim suffices to conclude.
	\begin{claim}
		If $\atmtn$ does not have value~$1$ w.r.t.\ $\finalStates$, then there is no positively winning strategy in $\pomdp^\atmtn$ for $\Parity(\pri)$.
		If $\atmtn$ has value~$1$ w.r.t.\ $\finalStates$, then there is an almost-sure strategy in $\pomdp^\atmtn$ for $\Parity(\pri)$.
	\end{claim}
	We prove the claim.
	We first assume that $\atmtn$ does not have value~$1$ w.r.t.~$\finalStates$.
	Let $\alpha > 0$ such that, for all $\strat\in\actions^*$, $\belief^\atmtn_\strat(\finalStates) \le 1 - \alpha$.
	To be an almost-sure strategy in $\pomdp^\atmtn$, a strategy needs to almost surely play~$\choice$ when the belief is a subset of $\states^\atmtn$ infinitely often; otherwise, there is a positive probability to only see priority~$1$.
	Whenever $\choice$ is played from a state in $\states^\atmtn$, let us consider the actions played since the last visit to $\state_0^\atmtn$: it is a word $\strat\in(\actions^\atmtn)^*$.
	Since $\belief^\atmtn_\strat(\finalStates) \le 1 - \alpha$, we have that the probability to visit $\state_3$ at the next step when playing $\choice$ is $\ge \alpha$.
	Therefore, a strategy in $\pomdp^\atmtn$ that almost surely plays~$\choice$ infinitely often when the belief is a subset of $\states^\atmtn$ will almost surely visit $\state_3$ infinitely often, and almost surely lose since $\pri(\state_3) = 3$.

	Let us now assume that $\atmtn$ has value~$1$ w.r.t.\ $\finalStates$.
	Therefore, for all $n\ge 1$, there is a word $\strat_n'\in(\actions^\atmtn)^*$ such that $\belief^\atmtn_{\strat_n'}(\finalStates) \ge 1 - \frac{1}{2^n}$.
	The proof carries out as for Example~\ref{ex:incompleteCE}.
	Let us divide a play in this POMDP into rounds $1$, $2$,\ldots; every time we go back to $\state_0$ after visiting $\state_2$ or $\state_3$, we move to the next round.
	We define a strategy $\strat_n$ that tries to play $\strat'_n$ from $\initState^\atmtn$ while staying in $\states^\atmtn$: if it fails to do so (i.e., it sees signal $\sig_0$ which means that it is back to~$\initState$), it goes back to $\initState^\atmtn$ and retries.
	Whenever $\strat'_n$ could be fully played while staying in $\states^\atmtn$, it plays $\choice$.

	Consider the strategy $\strat$ that plays $\strat_{n}$ in round $n$; we show that $\strat$ is almost sure.
	This strategy ensures that infinitely many rounds happen, because at each round~$n$, it will eventually succeed in playing $\strat'_n$ fully.
	At each round $n$, $\choice$ is eventually played with probability~$1$.
	When $\choice$ is played in round $n$, $\state_3$ is visited with probability $\le \frac{1}{2^n}$ and $\state_2$ is visited with probability $\ge 1 - \frac{1}{2^n}$.
	State $\state_2$ is clearly seen infinitely often almost surely, as the probability it is seen at each round is lower bounded by $\frac{1}{2}$.
	However, the probability that $\state_3$ is never seen anymore after round $n$ is greater than $\prod_{i = n}^\infty (1 - \frac{1}{2^i})$, which is positive and increases as $n$ grows to $\infty$.
	We deduce that the probability that $\state_3$ is seen at most finitely often is~$1$.
\end{proof}

\begin{remark}
	Although not our focus, a similar reduction shows that the value-$1$ problem for \emph{reachability} objectives in weakly revealing POMDPs is undecidable.
	To see it, consider the POMDP~$\pomdp'$ which is like $\pomdp^\atmtn$ except that $\state_2$ and $\state_3$ are absorbing and revealing.
	POMDP~$\pomdp'$ is still weakly revealing.
	Consider the reachability objective $\Reach(\state_2)$.
	We can similarly show that $\sup_{\strat\in\strats{\pomdp'}} \prob{\strat}{\pomdp'}{\Reach(\state_2)} = 1$ if and only if~$\atmtn$ has value~$1$ w.r.t.\ $\finalStates$.
\end{remark}

\section{Additional details for Section~\ref{sec:weaklyRevealing}: complexity of the weakly revealing property}
\label{app:weaklyRevealingProperty}

The goal in this section is to prove that deciding the weakly revealing property is \EXPTIME-complete (Theorem~\ref{thm:complexityWeak}).

We first give a direct proof that it is \EXPTIME-hard be reducing from the existence of a positive strategy for safety objectives in POMDPs.

\begin{lemma} \label{lem:complexityWeakLowerBound}
	Deciding whether a POMDP is weakly revealing is \textsf{EXPTIME}-hard.
\end{lemma}
\begin{proof}
	To show that this problem is \textsf{EXPTIME}-hard, we reduce from the existence of a positively winning strategy for safety objectives in POMDPs, which is \textsf{EXPTIME}-complete~\cite{Chatterjee.Chmelik.ea:2016}.

	Let $\pomdp = \pomdpFull$ be a POMDP with a safety objective $\Safety(\safe)$, where $\safe \subseteq \states$.
	The general idea of the proof is to build a POMDP $\pomdp'$ in which a revelation (and actually, infinitely many revelations) occur with probability~$1$ for all strategies if and only if there is no positively winning strategy in $\pomdp$ for $\Safety(\safe)$.
	We apply two transformations to $\pomdp$.
	\begin{itemize}
		\item First, we make two parallel, indistinguishable copies of the state space.
		This ensures that no revelations can happen, as there is no mechanism to know in which copy we currently are.
		This prevents existing revelations in~$\pomdp$ to alter our reduction.
		\item Second, we add a single sink state $\sink$ to which we redirect all transitions outgoing from $\safe$.
		We make sure that this sink state is revealed through a dedicated signal whenever it is reached.
		This ensures that infinitely many revelations happen if and only if $\safe$ is reached.
	\end{itemize}

	Formally, let $\pomdp' = (\states', \actions, \signals', \transitions', \initState')$, where
	\begin{itemize}
		\item $\states' = (\states \times \{1, 2\}) \disjUnion \{\initState', \sink\}$ ($\disjUnion$ denotes a \emph{disjoint} union),
		\item $\signals' = \signals \disjUnion \{\sig_0, \sig_\sink\}$,
		\item for all $\act\in\actions$, $\transitions'(\initState', a)(\sig_0, (\initState, 1)) = \transitions'(\initState', a)(\sig_0, (\initState, 2)) = \frac{1}{2}$ and for $\state\in\safe$, $i\in\{1, 2\}$, $\transitions'((\state, i), a)(\sig_\sink, \sink) = \transitions'(\sink, a)(\sig_\sink, \sink) = 1$.
		All other transitions are copied from $\pomdp$ and stay within their own copy of $\pomdp$.
	\end{itemize}
	We show that $\pomdp'$ is weakly revealing if and only if $\pomdp$ has no positively winning strategy for $\Safety(\safe)$.

	By construction, infinitely many revelations happen in $\pomdp'$ if and only if $\safe$ is reached.
	Therefore, $\pomdp'$ is weakly revealing if and only if for all strategies $\strat\in\strats{\pomdp'}$, $\prob{\strat}{\pomdp'}{\Reach(\reach \times \{1, 2\})} = 1$.
	If we restrict our focus to histories that have not gone through $\reach$, there is a natural bijection between strategies of $\pomdp'$ and of $\pomdp$.
	Based on this, one can show that for all strategies $\strat\in\strats{\pomdp'}$, $\prob{\strat}{\pomdp'}{\Reach(\reach \times \{1, 2\})} = 1$ if and only if for all strategies $\strat\in\strats{\pomdp}$, $\prob{\strat}{\pomdp}{\Reach(\reach)} = 1$.
	This last property says exactly that there is no positively winning strategy for $\Safety(\safe)$ in $\pomdp$, ending the proof.
\end{proof}

We now focus on the \EXPTIME upper bound.
We first focus on a related subproblem: given a set of belief supports, is there a strategy with a positive probability to avoid visiting any of them?
This is a kind of a safety objective on a POMDP; to the best of our knowledge, such safety objectives have been considered for sets of \emph{states}, but not for sets of \emph{belief supports}.
We first give an exponential-time algorithm to decide the above problem, and then relate it to the weakly revealing property.

Let $\pomdp = \pomdpFull$ be a POMDP, and $B\subseteq\powerSetNonEmpty{\states}$ be a set of belief supports (that we will want to avoid with a positive probability).
Given $\beliefSupp\in\powerSetNonEmpty{\states}\setminus B$, we define the event $\lnot B \until \beliefSupp$ (where we borrow $\until$ from the LTL ``until'' operator) as the set of plays $\set{\initState\act_1\sig_1\state_1\act_2\sig_2\ldots \mid \exists i\ge 1, \beliefUpd(\set{\initState}, \act_1\sig_1\ldots\act_i\sig_i) = \beliefSupp\ \textnormal{and}\ \forall j < i, \beliefUpd(\set{\initState}, \act_1\sig_1\ldots\act_j\sig_j)\notin B}$.
We say that a belief support $\beliefSupp\notin B$ is \emph{$B$-safely reachable} if there exists $\strat\in\strats{\pomdp}$ such that $\prob{\strat}{\pomdp}{\lnot B \until \beliefSupp} > 0$.

We consider a two-player deterministic game $\game_\pomdp^B = (\states_\game, \states_\game \times \actions, \edges)$ (with no distinguished initial state) derived from a POMDP $\pomdp$.
Intuitively, the states in $\states_\game$ are pairs of belief supports, where the second belief support is included in the first one.
From a state $(\beliefSupp_1, \beliefSupp_2)$ in $\states_\game$, \Pone can choose any action $\act$ in $\actions$, which moves the game to the state $(\beliefSupp_1, \beliefSupp_2, \act)$ controlled by \Ptwo.
From there, \Ptwo can resolve the probabilistic transition by choosing a signal that could have occurred when playing $\act$ \emph{from belief support $\beliefSupp_2$}.
Then, both belief supports are updated in a standard way following function $\beliefUpd$.
When the belief support along the first component is in $B$, we move to the losing sink $\bot$.

Formally, we define
\begin{align*}
\states_\game =\, &\set{(\beliefSupp_1, \beliefSupp_2) \mid \beliefSupp_1 \in\powerSetNonEmpty{\states} \setminus B\ \text{and}\ \beliefSupp_2\subseteq\beliefSupp_1} \cup \set{\bot}, \\
\edges =\, &\set{((\beliefSupp_1, \beliefSupp_2), (\beliefSupp_1, \beliefSupp_2, \act)\mid (\beliefSupp_1, \beliefSupp_2)\in\states_\game, \act\in\actions} \\
 &\cup \set{((\beliefSupp_1, \beliefSupp_2, \act), (\beliefSupp_1', \beliefSupp_2'))\mid \exists \sig\in\signals, \beliefUpd(\beliefSupp_1, \act, \sig) = \beliefSupp_1' \notin B\ \textnormal{and}\ \beliefUpd(\beliefSupp_2, \act, \sig) = \beliefSupp_2' \neq \emptyset} \\
 &\cup \set{((\beliefSupp_1, \beliefSupp_2, \act), \bot)\mid \exists \sig\in\signals, \beliefUpd(\beliefSupp_1, \act, \sig) \in B\ \textnormal{and}\ \beliefUpd(\beliefSupp_2, \act, \sig) \neq \emptyset} \\
 &\cup \set{(\bot, \bot)}. 
\end{align*}

Observe that the size of $\game_\pomdp^B$ is exponential in the size of $\pomdp$: the state space of $\game_\pomdp^B$ is upper-bounded by $2^{\card{\states}}\cdot 2^{\card{\states}} \cdot (\card{\actions} + 1) + 1$, and the number of edges by $2^{\card{\states}}\cdot 2^{\card{\states}} \cdot \card{\actions}\cdot(\card{\signals} + 1) + 1$.
We now show the relationship between the $\Safety(B)$ objective in $\pomdp$ and the $\Safety(\bot)$ objective in $\game_\pomdp^B$.

\begin{lemma} \label{lem:safetyBeliefSupports}
	The following are equivalent:
	\begin{enumerate}
		\item there is a strategy $\strat\in\strats{\pomdp}$ such that $\prob{\strat}{\pomdp}{\Safety(B)} > 0$,
		\item there exists a $B$-safely reachable belief support $\beliefSupp\in\powerSetNonEmpty{\states}$ and a state $\state\in\beliefSupp$ such that \Pone has a winning strategy in $\game_\pomdp^B$ for $\Safety(\bot)$ from $(\beliefSupp, \set{\state})$.
	\end{enumerate}
\end{lemma}
\begin{proof}
	We first show the implication from $2$.\ to $1$.
	We assume that there exists a $B$-safely reachable belief support $\beliefSupp\in\powerSetNonEmpty{\states}$ and a state $\state\in\beliefSupp$ such that \Pone has a winning strategy $\strat_{(\beliefSupp, \set{\state})}$ in $\game_\pomdp^B$ for $\Safety(\bot)$ from $(\beliefSupp, \set{\state})$.
	We build a strategy $\strat\in\strats{\pomdp}$ such that $\prob{\strat}{\pomdp}{\Safety(B)} > 0$.
	First, since~$\beliefSupp$ is $B$-safely reachable, there is a strategy $\strat_\beliefSupp\in\strats{\pomdp}$ such that $\prob{\strat_\beliefSupp}{\pomdp}{\lnot B \until \beliefSupp} > 0$.
	The strategy $\strat$ first plays like $\strat_\beliefSupp$ as long as belief support $\beliefSupp$ is not reached.
	Whenever $\beliefSupp$ is reached (which happens with positive probability), there is a positive probability that $\state$ is the actual state reached.
	Then, $\strat$ moves to the strategy $\strat_{(\beliefSupp, \set{\state})}$ that wins surely in $\game_\pomdp$ for $\Safety(\bot)$ from $(\beliefSupp, \state)$.
	Notice that if~$\state$ is the actual state reached at the first visit in $\beliefSupp$, any point after $\beliefSupp$ is reached, the belief support in $\pomdp$ corresponds to a belief support that is reachable in $\game_\pomdp^B$ following $\strat_{(\beliefSupp, \set{\state})}$.
	Therefore, if $\state$ is the actual state reached at the first visit in $\beliefSupp$, this strategy avoids any belief support in $B$ no matter how the probabilistic transitions are resolved.
	This leads to a positive probability of avoiding any belief support in $B$ by playing $\strat$.

	We now show the implication from $1$.\ to $2$.
	We prove the contrapositive.
	Assume that for all $B$-safely reachable belief supports $\beliefSupp\in\powerSetNonEmpty{\states}$ and $\state\in\beliefSupp$, \Pone has no winning strategy in $\game_\pomdp^B$ for $\Safety(\bot)$ from $(\beliefSupp, \set{\state})$.
	We show that for all strategies $\strat\in\strats{\pomdp}$, $\prob{\strat}{\pomdp}{\Safety(B)} = 0$.

	Let $\strat\in\strats{\pomdp}$ be any strategy in $\pomdp$.
	We show that as long as $B$ is not reached, the probability to reach $B$ is lower-bounded by a uniform positive probability.
	Let $\hist$ be an observable history that has not yet reached a belief support in $B$.
	This means that $\beliefSupp = \beliefUpdStar(\set{\initState}, \hist)$ is $B$-safely reachable.
	Let $\belief\in\dist{\states}$ be the corresponding belief after $\hist$ occurred (in particular, $\supp{\belief} = \beliefSupp$).
	Let $\state\in\beliefSupp$ be any possible state after this history.
	By our hypothesis, there is no winning strategy for \Pone in $\game_\pomdp^B$ for $\Safety(\bot)$ from $(\beliefSupp, \set{\state})$.
	As two-player deterministic reachability games are determined, this means that \Ptwo has a strategy that ensures $\Reach(\bot)$.
	Classical results on reachability games (using the \emph{attractor decomposition}~\cite{Bloem.Chatterjee.ea:2018}) imply that \Ptwo even has such a strategy ensuring that no state is visited twice along the path before reaching state $\bot$ of $\game_\pomdp^B$.
	Hence, \Ptwo has a strategy ensuring that \Pone plays at most $\card{\states_\game}$ times before visiting $\bot$.
	By construction of the game, this path can happen with positive probability in the POMDP if $\beliefSupp$ is the current belief support and $\state$ is the current state, so we find that $\prob{\strat}{\pomdp}{\Reach^{\le \card{\hist} + \card{\states_\game}}(B)\mid \hist\state} \ge \leastProb_\pomdp^{\card{\states_\game}}$.
	Hence, by conditioning over all states in $\beliefSupp$, we have
	\begin{align*}
	\prob{\strat}{\pomdp}{\Reach^{\le \card{\hist} + \card{\states_\game}}(B)\mid \hist} &= \sum_{\state\in\beliefSupp} \prob{\strat}{\pomdp}{\Reach^{\le \card{\hist} + \card{\states_\game}}(B)\mid \hist\state}\cdot \belief(\state) \\
	&\ge \sum_{\state\in\beliefSupp} \leastProb_\pomdp^{\card{\states_\game}}\cdot \belief(\state) \\
	&\ge \leastProb_\pomdp^{\card{\states_\game}}.
	\end{align*}
	Hence, as long as $B$ has not been reached yet, there is always a probability $\ge \leastProb_\pomdp^{\card{\states_\game}}$ to visit $B$ within the next $\card{\states_\game}$ steps.
	Hence, $B$ is reached with probability~$1$.
\end{proof}

\begin{theorem} \label{thm:safetyBeliefSupportsComplexity}
	The following problem is in \textsf{EXPTIME}: given a POMDP $\pomdp = \pomdpFull$ and a set of belief supports $B\subseteq\powerSetNonEmpty{\states}$, decide whether there exists $\strat\in\strats{\pomdp}$ such that $\prob{\strat}{\pomdp}{\Safety(B)} > 0$.
\end{theorem}
\begin{proof}
	We first build the exponential-sized game $\game_\pomdp^B$ in exponential time.
	Note that a belief support $\beliefSupp$ is $B$-safely reachable if and only if there is a reachable state $(\beliefSupp, \beliefSupp')$ in $\game_\pomdp$ from $(\set{\initState}, \set{\initState})$.
	Then, we compute the winning region for objective $\Safety(\bot)$; using classical results on reachability games~\cite{Bloem.Chatterjee.ea:2018}, this can be done in time linear in the size of the game, hence in exponential time in our case.

	We then use Lemma~\ref{lem:safetyBeliefSupports} to conclude: there exists $\strat\in\strats{\pomdp}$ such that $\prob{\strat}{\pomdp}{\Safety(B)} > 0$ if and only if there is a state $(\beliefSupp, \set{\state})$ in the winning region of \Pone in $\game_\pomdp^B$.
\end{proof}

We deduce our desired result about the complexity of the weakly revealing property.

\complexityWeak*
\begin{proof}
	The \EXPTIME-hardness follows from Lemma~\ref{lem:complexityWeakLowerBound}.

	To show the membership in \EXPTIME, we reduce to the decision problem from Theorem~\ref{thm:safetyBeliefSupportsComplexity}. We show that the following are equivalent:
	\begin{enumerate}
		\item $\pomdp$ is \emph{not} weakly revealing,
		\item there is a reachable singleton belief support $\set{\state}$ and a strategy $\strat\in\strats{\pomdp^\state}$ such that $\prob{\strat}{\pomdp^\state}{\Safety^{\ge 1}(\singletons^\pomdp)} > 0$ (i.e., that never revisits any singleton belief support after $\set{\state}$ with positive probability).
	\end{enumerate}
	Clearly, given $2$., there is a strategy that reaches finitely many singleton belief supports with positive probability: first reach $\set{\state}$ with positive probability, and then play $\strat$ from there. So $\pomdp$ is not weakly revealing.

	We now show the contrapositive of the implication from $1$.\ to $2$.
	Assume that for all singleton belief supports $\set{\state}$ and strategies $\strat\in\strats{\pomdp^\state}$, we have $\prob{\strat}{\pomdp^\state}{\Safety^{\ge 1}(\singletons^\pomdp)} = 0$.
	This means that from any singleton belief support, any strategy has probability~$1$ to revisit another singleton belief support.
	All strategies therefore visit infinitely many singleton belief supports with probability~$1$, so $\pomdp$ is weakly revealing.

	An \EXPTIME algorithm follows from this characterization: compute all the reachable singleton belief supports, and query whether there exists a strategy $\strat\in\strats{\pomdp^\state}$ such that $\prob{\strat}{\pomdp^\state}{\Safety^{\ge 1}(\singletons^\pomdp)} > 0$.
	By Theorem~\ref{thm:safetyBeliefSupportsComplexity}, every such query can be done in exponential time.
\end{proof}

\section{Additional details for Section~\ref{sec:stronglyRevealing}}
\label{app:stronglyRevealing}

In this section, we give missing proofs of statements from Section~\ref{sec:stronglyRevealing}.
We start with the complexity lower bound on the existence of almost-sure strategies.

\exptimeHardnessRevealing*
\begin{proof}
	Our proof is by reduction from the existence of an almost-sure strategy for safety objectives in general POMDPs, which is \EXPTIME-complete~\cite{Chatterjee.Chmelik.ea:2016}.

	Let $\pomdp = \pomdpFull$ be a POMDP along with an objective $\Safety(\safe)$ for some $\safe\subseteq \states$.
	We define a new POMDP $\pomdp' = (\states, \actions, \signals', \transitions', \initState)$ on the same state and action space, in which we split each transition of $\pomdp$ into three transitions: the original transition, a transition with a dedicated signal revealing the target state, and a transition resetting to the initial state, also with a signal indicating the reset.
	Formally, let
	\begin{itemize}
		\item $\signals' = \signals \disjUnion \{\sig_\state \mid \state\in\states\} \disjUnion \{\sig_\mathsf{reset}\}$,
		\item \emph{copy of the original transition}: for all $\state, \state'\in\states$, $\act\in\actions$, $\sig\in\signals$, we define $\transitions'(\state, \act)(\sig, \state') = \frac{\transitions(\state, \act)(\sig, \state')}{3}$,
		\item \emph{revelations over all possible states}: for all $\state\in\states$ and $\act\in\actions$, let $\Succ(\state, \act) = \set{\state'\mid \exists \sig\in\signals, \transitions(\state, \act)(\sig, \state') > 0}$ be the set of states that can be reached from $\state$ playing action $\act$ in one step; for $\state'\in\Succ(\state, \act)$, we define $\transitions'(\state, \act)(\sig_{\state'}, \state') = \frac{1}{3\cdot\card{\Succ(\state, \act)}}$,
		\item \emph{reset}: for all $\state\in\states$, $\act\in\actions$, we define $\transitions'(\state, \act)(\sig_{\mathsf{reset}}, \initState) = \frac{1}{3}$,
	\end{itemize}
	POMDP $\pomdp'$ is strongly revealing: for every transition, there is a corresponding transition revealing the target state.
	The CoB\"uchi objective we define is the one induced by the priority function $\pri$ such that $\pri(\state) = 1$ if $\state\in\safe$, and $\pri(\state) = 0$ if $\state\notin\safe$.

	We show that there is an almost-sure strategy for the CoB\"uchi objective $\Parity(\pri)$ in $\pomdp'$ if and only if there is an almost-sure strategy for $\Safety(\safe)$ in $\pomdp$.

	Assume first that there is an almost-sure strategy $\strat$ for $\Safety(\safe)$ in $\pomdp$.
	The idea is to try to play $\strat$ in $\pomdp'$ to avoid~$\safe$, but there are two new events to take into account:
	\begin{itemize}
		\item whenever signal $\sig_\mathsf{reset}$ is seen, simply replay $\strat$ from the start;
		\item whenever a revealing signal $\sig_{\state'}$ is seen on a transition $\transitions(\state, \act)(\sig_{\state'}, \state') > 0$, simply assume that the signal seen was actually a signal $\sig\in\signals$ such that $\transitions(\state, \act)(\sig, \state') > 0$ (which exists by construction of $\pomdp'$).
		This way, the assumed belief support is just an overapproximation of the actual singleton belief support.
		As $\strat$ wins almost surely for $\Safety(\safe)$ in $\pomdp$, even the overapproximated belief support will never contain any state in $\safe$.
	\end{itemize}
	The strategy built is actually almost sure for $\Safety(\safe)$ in $\pomdp'$, so it is in particular almost sure for $\Parity(\pri)$.

	Assume now that there is no almost-sure strategy for $\Safety(\safe)$ in $\pomdp$.
	Then, by Lemma~\ref{lem:pomdpReach}, there is $n\in\IN$ and $\alpha > 0$ such that for all strategies $\strat\in\strats{\pomdp}$, $\prob{\strat}{\pomdp}{\Reach^{\le n}(\safe)} \ge \alpha$.

	Consider now any strategy $\strat'$ on $\pomdp'$.
	Almost surely, infinitely many resets (with signal $\sig_\mathsf{reset}$) happen.
	Also almost surely, infinitely often, after each reset, the game lasts for more than $n$ steps without any revelation or reset.
	Under this condition, the probability to visit $\safe$ is at least $\alpha$.
	Hence, there is almost surely and infinitely often a lower-bounded probability to visit $\safe$, so $\safe$ is almost surely visited infinitely often.
	Hence, there is no almost-sure strategy for $\Parity(\pri)$ in $\pomdp$.
\end{proof}

We now show that the analysis of the belief-support MDP is complete for strongly revealing POMDPs.

\strongComplete*
\begin{proof}
	Let $\strat\in\strats{\pomdp}$ be an almost-sure strategy for $\Parity(\pri)$ in $\pomdp$.

	Let $\ECs_{\strat}^\pomdp = \{\ec = (\ecStates, \ecActions) \mid \prob{\strat}{\pomdp}{\inf(\play) = U} > 0\}$ be the end components that $\strat$ can end up in, and visit exactly all their states and actions infinitely often, in the underlying MDP of~$\pomdp$.
	We build an almost-sure strategy in $\beliefMDP$.
	By Lemma~\ref{lem:bijectionStrategiesAS}, there is a strategy on $\beliefMDP$ that almost surely reaches a singleton belief support from some end component in $\ECs_{\strat}^\pomdp$.
	We now define what the strategy does after such a singleton belief is reached.

	Let $\ec = (\ecStates, \ecActions) \in \ECs_{\strat}^\pomdp$.
	As it is a possible end component of $\strat$, $p_\mathsf{max} = \max \pri(\ec)$ is even.
	Let $\state_\mathsf{max}^\ec \in \ecStates$ such that $\pri(\state_\mathsf{max}^\ec) = p_\mathsf{max}^\ec$.
	As an end component is strongly connected, from any state $\state\in \ecStates$, there is a history from $\state$ to~$\state_\mathsf{max}^\ec$.
	We denote it $\hist_\state^\ec = \state_0 \xrightarrow{\act_1} \state_1 \xrightarrow{\act_2} \dots \xrightarrow{\act_n} \state_n$, where $\state_0 = \state$ and $\state_n = \state_\mathsf{max}^\ec$.

	Whenever any belief support singleton $\{\state\}$ for some $\state\in \ecStates$ is reached, we play the following strategy on~$\beliefMDP$:
	\begin{itemize}
		\item we try to achieve exactly the history $\hist_\state^\ec$, hoping for a revelation after \emph{every} action (this exploits that $\pomdp$ is strongly revealing);
		\item if that fails, let $B_{\lnot \ecStates} = \{\beliefSupp\in\powerSetNonEmpty{\pomdp} \mid b \cap (\states \setminus \ecStates) \neq \emptyset\}$, i.e., all the belief supports that indicate that $\ecStates$ may have been left.
		We show below that we can play an almost-sure strategy for $\Safety(B_{\lnot \ecStates})$.
		We play this strategy until we fall back to a singleton belief support in $U$, which happens eventually due to $\pomdp$ being strongly revealing.
	\end{itemize}

	This strategy is almost sure for $\Parity(\priMDP)$: eventually, it reaches a singleton belief support from some end component $\ec$ in $\ECs_{\strat}^\pomdp$.
	Then, it only sees beliefs whose states are all in~$\ec$.
	And whenever a singleton belief support is reached, which happens infinitely often, we have a lower-bounded probability to reach $\state_\mathsf{max}^\ec$, so $\state_\mathsf{max}^\ec$ is reached infinitely often almost surely.
	It remains to show that whenever we deviate from a revelation at every step of a history~$\hist_\state^\ec$, we have an almost-sure strategy for $\Safety(B_{\lnot \ecStates})$.

	Let $\ec = (\ecStates, \ecActions)\in\ECs_{\strat}^\pomdp$, $\state\in\ecStates$, and $\hist_\state^\ec = \state_0 \xrightarrow{\act_1} \state_1 \xrightarrow{\act_2} \dots \xrightarrow{\act_n} \state_n$.
	Let $\beliefSupp\in \powerSetNonEmpty{\states}$ be such that $\transitionsMDP(\{\state_i\}, \act_{i+1})(\beliefSupp) > 0$ (i.e., $\beliefSupp$ is a possible successor of $\{\state_i\}$ along the path $\hist_\state$).
	We show that there is an almost-sure strategy from $\beliefSupp$ for $\Safety(B_{\lnot \ecStates})$ in $\beliefMDP$.

	Assume on the contrary that for all strategies~$\strat'\in\strats{\beliefMDP^\beliefSupp}$, we have $\prob{\strat'}{\beliefMDP^\beliefSupp}{\Reach(B_{\lnot \ecStates})} > 0$.
	Then, by Lemma~\ref{lem:mdpReach}, for all strategies~$\strat'\in\strats{\beliefMDP^\beliefSupp}$, we have $\prob{\strat'}{\beliefMDP^\beliefSupp}{\Reach^{\le 2^{\card{\states}}-1}(B_{\lnot \ecStates})} > 0$.
	Hence, by Lemma~\ref{lem:bijectionStrategies}, for all strategies $\strat\in\strats{\pomdp^\beliefSupp}$, $\prob{\strat}{\pomdp^\beliefSupp}{\Reach^{\le 2^{\card{\states}}-1}(B_{\lnot \ecStates})} > 0$.
	This means that for all strategies, there is a history of length at most $2^{\card{\states}}-1$ that exits $\ecStates$ from $\beliefSupp$.
	For pure strategies, the probability to exit $\ecStates$ is therefore at least $\leastProb_\pomdp^{2^{\card{\states}}-1}$; by Theorem~\ref{thm:pureSuffice}, this extends to all (not only pure) strategies.
	This means that every time~$\set{\state_i}$ is visited and $\act_{i+1}$ is played, there is a lower-bounded probability to exit $R$.
	This contradicts that $U$ is an end component.
\end{proof}

To conclude the missing proofs, we discuss our undecidability result for strongly revealing CoB\"uchi games.
The syntax for games was defined in Section~\ref{sec:games}.

\begin{remark}
	The model of~\cite{Bertrand.Genest.ea:2017} allows for distinct signals for both players, and is therefore slightly more general.
	Our undecidability proof works even when the signals given to both players are the same, which is why we opted for this restriction.
\end{remark}

\undecidableGames*
\begin{proof}
	We reduce from the undecidable \emph{value-$1$ problem for probabilistic automata}.
	This problem was already used in the reduction of Theorem~\ref{thm:undecidableParity}; we refer to Appendix~\ref{app:weaklyRevealing} for a definition and an introduction to the problem.

	Let $\atmtn = \atmtnFull$ be a probabilistic automaton and $\finalStates\subseteq \states^\atmtn$ be a set of final states.
	We build a game $\game^\atmtn$ with a CoB\"uchi objective such that \Pone wins almost surely if and only if $\atmtn$ does not have value~$1$ w.r.t.\ $\finalStates$.

	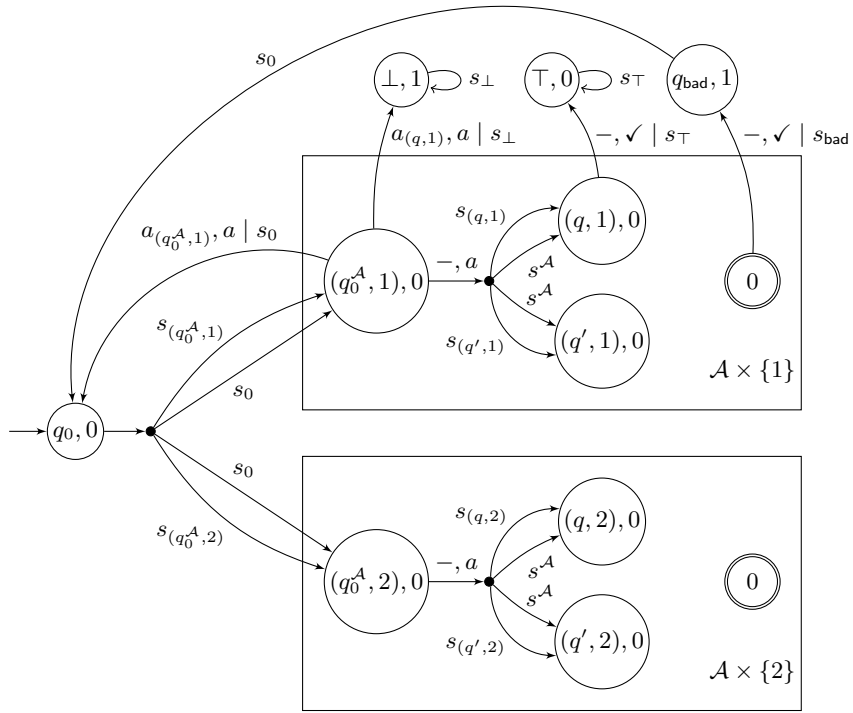
\begin{figure}
		\centering
		\begin{tikzpicture}[every node/.style={font=\small,inner sep=1pt}]
			\draw ($(0,0)$) node[rond] (qinit) {$\initState, 0$};
			\draw ($(qinit.west)-(0.5,0)$) edge[-latex'] (qinit);
			\draw ($(qinit)+(1,0)$) node[dot] (qmid) {};

			\draw ($(qinit)+(4,2)$) node[rond] (q1) {$(\initState^\atmtn, 1), 0$};
			\draw ($(q1)+(1.5,0)$) node[dot] (qmidA) {};
			\draw (q1) edge[-latex'] node[above=2pt] {$-, a$} (qmidA);
			\draw ($(q1)+(3,.8)$) node[rond] (q2) {$(\state, 1), 0$};
			\draw ($(q1)+(3,-.8)$) node[rond] (q1') {$(\state', 1), 0$};
			\draw ($(q1)+(5,0)$) node[rond,accepting] (F) {$0$};
			\draw (qmidA) edge[-latex',bend left=10] node[below right] {$\sig^\atmtn$} (q2);
			\draw (qmidA) edge[-latex',bend right=10] node[above right] {$\sig^\atmtn$} (q1');
			\draw (qmidA) edge[-latex',bend left=45] node[above left] {$\sig_{(\state, 1)}$} (q2);
			\draw (qmidA) edge[-latex',bend right=45] node[below left] {$\sig_{(\state', 1)}$} (q1');
			\node[rectangle, fit=(q1) (q2) (q1') (F), draw=black,inner sep=8pt] (A) {};
			\draw ($(F)+(0,-1.2)$) node[] {$\atmtn\times\{1\}$};

			\draw ($(qinit)+(4,-2)$) node[rond] (q1B) {$(\initState^\atmtn, 2), 0$};
			\draw ($(q1B)+(1.5,0)$) node[dot] (qmidB) {};
			\draw (q1B) edge[-latex'] node[above=2pt] {$-, a$} (qmidB);
			\draw ($(q1B)+(3,.8)$) node[rond] (q2B) {$(\state, 2), 0$};
			\draw ($(q1B)+(3,-.8)$) node[rond] (q1'B) {$(\state', 2), 0$};
			\draw ($(q1B)+(5,0)$) node[rond,accepting] (FB) {$0$};
			\draw (qmidB) edge[-latex',bend left=10] node[below right] {$\sig^\atmtn$} (q2B);
			\draw (qmidB) edge[-latex',bend right=10] node[above right] {$\sig^\atmtn$} (q1'B);
			\draw (qmidB) edge[-latex',bend left=45] node[above left] {$\sig_{(\state, 2)}$} (q2B);
			\draw (qmidB) edge[-latex',bend right=45] node[below left] {$\sig_{(\state', 2)}$} (q1'B);
			\node[rectangle, fit=(q1B) (q2B) (q1'B) (FB), draw=black,inner sep=8pt] (B) {};
			\draw ($(FB)+(0,-1.2)$) node[] {$\atmtn\times\{2\}$};

			\draw ($(A)+(0,2.7)$) node[rond] (top) {$\top, 0$};
			\draw ($(A)+(-2,2.7)$) node[rond] (bot) {$\bot, 1$};
			\draw ($(A)+(2,2.7)$) node[rond] (bad) {$\state_\mathsf{bad}, 1$};
			\draw (top) edge[-latex',loop right] node[right=2pt] {$\sig_\top$} (top);
			\draw (bot) edge[-latex',loop right] node[right=2pt] {$\sig_\bot$} (bot);
			\draw (bad) edge[-latex',bend right=66] node[above left] {$\sig_0$} (qinit);

			\draw (qinit) edge[-latex'] node[above] {} (qmid);
			\draw (qmid) edge[-latex'] node[below=5pt] {$\sig_0$} (q1);
			\draw (qmid) edge[-latex'] node[above=5pt] {$\sig_0$} (q1B);
			\draw (qmid) edge[-latex',bend left=20] node[above left,xshift=1pt,yshift=-1pt] {$\sig_{(\initState^\atmtn, 1)}$} (q1);
			\draw (qmid) edge[-latex',bend right=20] node[below left,xshift=1pt,yshift=-1pt] {$\sig_{(\initState^\atmtn, 2)}$} (q1B);

			\draw (q2) edge[-latex',bend right=15] node[right,yshift=1pt,xshift=1pt] {$-, \checkmark \mid \sig_\top$} (top);
			\draw (F) edge[-latex',bend right=15] node[right,yshift=17pt,xshift=-3pt] {$-, \checkmark\mid \sig_{\mathsf{bad}}$} (bad);

			\draw (q1) edge[-latex',bend right=50] node[above,xshift=14pt,yshift=9pt] {$\act_{(\initState^\atmtn, 1)}, a \mid \sig_0$} (qinit);
			\draw (q1) edge[-latex',bend left=10] node[right,yshift=12pt,xshift=4pt] {$\act_{(\state, 1)}, a \mid \sig_\bot$} (bot);
		\end{tikzpicture}
		\caption{Game $\game^\atmtn$ used in the proof of \Cref{thm:undecidableGames}.
			States with a double circle correspond to states of $\finalStates$.
			Many transitions are not represented, but we illustrate at least one transition of each kind (randomization in the initial state, possible revelations for each state inside $\atmtn\times\{1, 2\}$, right and wrong $\finalStates$-guesses, right and wrong state guesses for \Pone).}
		\label{fig:undecidableGames}
	\end{figure}

	Our construction is illustrated in Figure~\ref{fig:undecidableGames}.
	We briefly give an intuition of the construction before a more formal definition.
	The game happens in a copy of $\atmtn$, in which all states are given priority~$0$.
	\Ptwo picks the letters in $\actions^\atmtn$ to induce transitions in $\atmtn$.
	To see priority~$1$, \Ptwo needs to ``$\finalStates$-guess'', i.e., claim that the current state is in~$\finalStates$ by playing a special action $\checkmark$: if the guess is right, \Ptwo goes to a special state $\state_\mathsf{bad}$ that produces priority~$1$, and the game resets; if the guess is wrong, \Ptwo immediately loses as the next state is a sink state $\top$ with priority~$0$.

	We artificially make the game strongly revealing by adding a special signal which may reveal the state at each transition.
	This makes the game easier for \Ptwo, as it is easier for \Ptwo to reach a state of $\finalStates$ if revelations happen.
	To balance the game, we give to \Pone the power to reset the game by guessing the current state; \Pone has a special action for each state.
	If the guess is correct (which is easy after a revelation), the game resets to the initial state; if the guess is wrong, a sink state $\bot$ with priority~$1$ is reached, which immediately loses for \Pone.
	\Pone may also choose not to guess anything by playing the waiting move $-$.
	To prevent revelations that may be occurring in $\atmtn$ from helping \Pone, we make two copies of $\atmtn$ and start the game randomly in one of the two copies, which are never distinguished except by the added revealing signals (the same trick was used in the proof of Lemma~\ref{lem:complexityWeakLowerBound}).
	This way, \Pone can only guess the current state with certainty after one of the added revealing signals occurred.
	We also give to \Ptwo the power to reset the game by correctly guessing the current state, which prevents a state that cannot reach $\finalStates$ anymore to be attained and revealed, with no way for \Ptwo to ever go back to $\finalStates$.

	Formally, we define the game $\game^\atmtn = (\states, \actions_1, \actions_2, \signals, \transitions, \initState)$ as follows:
	\begin{itemize}
		\item $\states = (\states^\atmtn\times\{1, 2\}) \disjUnion \set{\initState, \state_\mathsf{bad}, \bot, \top}$,
		\item $\actions_1 = \{\act_{(\state, i)}\mid(\state, i)\in\states^\atmtn\times\{1, 2\}\} \disjUnion \{-\}$,
		\item $\actions_2 = \{\act_{(\state, i)}\mid(\state, i)\in\states^\atmtn\times\{1, 2\}\} \disjUnion\actions^\atmtn\disjUnion\{\checkmark\}$,
		\item $\signals = \{\sig^\atmtn, \sig_0, \sig_\mathsf{bad}, \sig_\top, \sig_\bot\} \disjUnion \{\sig_{(\state, i)} \mid (\state, i)\in\states^\atmtn\times\{1, 2\}\}$,
		\item $\top$ and $\bot$ are sinks: for all $\act\in\actions_1\times\actions_2$, $\transitions(\top, \act)(\sig_\top, \top) = \transitions(\bot, \act)(\sig_\bot, \bot) = 1$,
		\item we give the following priorities to states: $\pri(\initState) = 0$, $\pri(\state_\mathsf{bad}) = 1$, $\pri(\top) = 0$, $\pri(\bot) = 1$, and for $(\state, i)\in\states \times \{1, 2\}$, $\pri((\state, i)) = 0$,
		\item after $\state_\mathsf{bad}$, we always move back to $\initState$: for $(\act_1, \act_2)\in\actions_1\times\actions_2$, $\transitions(\state_\mathsf{bad}, (\act_1, \act_2))(\sig_0, \initState) = 1$,
		\item the initial state randomizes over the two copies of the initial state of $\atmtn$ (there is already a positive probability of a revelation): for all $\act\in\actions_1\times\actions_2$, $\transitions(\initState, \act)(\sig_0, (\initState^\atmtn, 1)) = \transitions(\initState, \act)(\sig_0, (\initState^\atmtn, 2)) = \transitions(\initState, \act)(\sig_{(\initState^\atmtn, 1)}, (\initState^\atmtn, 1)) = \transitions(\initState, \act)(\sig_{(\initState^\atmtn, 2)}, (\initState^\atmtn, 2)) = \frac{1}{4}$,
		\item when \Pone plays $-$, transitions in each copy of $\atmtn$ behave like in $\atmtn$, with a positive probability of a revelation at each transition: for $(\state, i)\in\states^\atmtn\times\{1, 2\}$, $\state'\in\states^\atmtn$, and $\act\in\actions_2$, $\transitions((\state, i), (-, \act))(\sig^\atmtn, (\state', i)) = \transitions((\state, i), (-, \act))(\sig_{(\state', i)}, (\state', i)) = \frac{\transitions^\atmtn(\state, \act)(\state')}{2}$,
		\item at each round, players can guess the current state of $\states^\atmtn\times\{1, 2\}$ by playing the corresponding action: a wrong guess ends the game by leading to the sink state that is losing for the player guessing wrong, while a right guess simply resets the game.
		For a state $(\state, i)\in\states^\atmtn\times\{1, 2\}$, we say that action $\act_{(\state, i)}$ is a \emph{right guess}, while an action $\act_{(\state', j)}$ is a \emph{wrong guess} if $\state'\neq\state$ or $j \neq i$.
		Formally, for all $(\state, i)\in\states^\atmtn\times\{1, 2\}$, $\act_1\in\actions_1$, $\act_2\in\actions_2$, we define $\transitions((\state, i), (\act_1, \act_2))(\sig_\bot, \bot) = 1$ if $\act_1$ is a wrong guess, and $\transitions((\state, i), (\act_{(\state', j)}, \act_2))(\sig_\top, \top) = 1$ if $\act_2$ (but not $\act_1$) is a wrong guess.
		For all $(\state, i)\in\states^\atmtn\times\{1, 2\}$, $\act_1\in\actions_1$, $\act_2\in\actions_2$, we define $\transitions((\state, i), (\act_1, \act_2))(\sig_0, \initState) = 1$ if $\act_1$ or $\act_2$ is a right guess (and none is a wrong guess).
		\item If there are no other guesses, \Ptwo can ``$\finalStates$-guess'' whether the current state is in $\finalStates \times \{1, 2\}$ with action $\checkmark$: if it is a right $\finalStates$-guess, then $\state_\mathsf{bad}$ is reached, producing priority~$1$, and the game then resets.
		If it is a wrong $\finalStates$-guess, the next state is $\top$, an immediate win for \Pone.
		Formally, for all $(\state, i)\in\states^\atmtn\times\{1, 2\}$, we define $\transitions((\state, i), (-, \checkmark))(\sig_\mathsf{bad}, \state_\mathsf{bad}) = 1$ if $\state\in\finalStates$, and $\transitions((\state, i), (-, \checkmark))(\sig_\top, \state_\top) = 1$ if $\state\notin\finalStates$.
	\end{itemize}
	Game $\game^\atmtn$ is strongly revealing: every state has a dedicated signal, which is produced with non-zero probability on each incoming transition.
	To win, \Pone needs to avoid ending up in $\bot$ and avoid visiting $\state_\mathsf{bad}$ infinitely often.

	We show the following claim, which suffices to conclude.
	\begin{claim}
		\Pone has an almost-sure strategy for the CoB\"uchi objective of the strongly revealing game $\game^\atmtn$ if and only if $\atmtn$ does not have value~$1$ w.r.t.\ $\finalStates$.
	\end{claim}

	We prove the claim.
	Assume first that $\atmtn$ has value~$1$ w.r.t.~$\finalStates$.
	We show that \Ptwo has a positively winning strategy.
	As $\atmtn$ has value~$1$ w.r.t.\ $\finalStates$, there is a sequence $\strat_1, \strat_2, \ldots$ of words in $(\actions^\atmtn)^*$ such that $\belief^\atmtn_{\strat_i}(\finalStates) \ge 1 - \frac{1}{2^i}$ (notation $\belief^\atmtn_{\strat_i}$ was defined in Appendix~\ref{app:weaklyRevealing}).
	To define the strategy of \Ptwo, we split a play into rounds.
	The game starts at round $1$.
	At round $i$, \Ptwo tries to play word $\strat_i$ fully in a copy of $\atmtn$ without any revelation (i.e., a signal $\sig_{(\state, i)}$ for some $(\state, i)\in\states^\atmtn\times\{1, 2\}$); if such a revelation $\sig_{(\state, i)}$ happens, \Ptwo plays the corresponding action $\act_{(\state, i)}$ to reset the game and tries to play $\strat_i$ again.
	Once $\strat_i$ can be played fully without a revelation (which happens eventually with probability~$1$), \Ptwo plays $\checkmark$, which goes to $\state_\mathsf{bad}$ with probability $\ge 1 - \frac{1}{2^i}$.
	The strategy then moves over to round $i+1$.
	Assuming \Pone only plays $-$, this strategy wins for \Ptwo with probability greater than the infinite product $\prod_{i \ge 1} (1 - \frac{1}{2^i})$, which is a positive number.
	If \Pone tries a guess when in a copy of $\atmtn$, either a revelation just happened, in which case \Ptwo was resetting the game anyway, or the probability of a wrong guess for \Pone is at least $\frac{1}{2}$, which immediately wins the game with positive probability for \Ptwo.

	Assume now that $\atmtn$ does not have value~$1$ w.r.t.\ $\finalStates$.
	This means that there is $\alpha > 0$ such that for all finite words $\strat\in(\actions^\atmtn)^*$, $\belief^\atmtn_\strat(\finalStates) \le 1 - \alpha$.
	We define a strategy of \Pone: play a right guess whenever a revelation happens in $\states^\atmtn\times\{1,2\}$, and otherwise always play $-$.
	We show that this strategy is almost sure for \Pone.
	Observe that this strategy guarantees that $\bot$ is never reached and that the game resets infinitely often (unless the winning state $\top$ is reached due to a mistake of \Ptwo).
	To win, \Ptwo needs to visit $\state_\mathsf{bad}$ infinitely often, which requires to play~$\checkmark$ infinitely often while in a state of $\finalStates\times\{1, 2\}$.
	After every reset, two things can happen when the play moves into $\states^\atmtn\times\{1, 2\}$:
	\begin{itemize}
		\item either a revelation happens, in which case \Pone immediately resets the game;
		\item or \Ptwo plays some word $\strat\in(\actions^\atmtn)^*$ without any revelation and then attempts to play $\checkmark$.
		This reaches $\top$ with probability $\ge \alpha$.
	\end{itemize}
	Either \Ptwo plays $\checkmark$ only finitely often, in which case priority~$1$ is seen at most finitely often, or \Ptwo attempts $\checkmark$ infinitely often from a non-revealed state of $\states^\atmtn\times\{1, 2\}$, which eventually leads to $\top$ almost surely.
	In both cases, \Pone wins almost surely.
\end{proof}

\end{document}